\documentclass{article}

\usepackage{microtype}
\usepackage{graphicx}
\usepackage{xcolor}
\usepackage{booktabs} 

\usepackage{hyperref}

\usepackage[accepted]{icml2020}

\newcommand{\ignore}[1]{}

\usepackage{microtype}

\usepackage{booktabs}
\usepackage{multirow}
\usepackage{array}

\usepackage{url}

\usepackage{subcaption}

\usepackage{graphicx} 
\usepackage{tikz}
\usetikzlibrary{arrows.meta,calc,decorations.markings,math,arrows.meta}

\usepackage{amsmath, amssymb, amsthm, dsfont, amsfonts}
\usepackage{nicefrac}       
\usepackage{mathtools}

\newcommand\numberthis{\addtocounter{equation}{1}\tag{\theequation}}  

\theoremstyle{definition}

\newtheorem{theorem}{Theorem}[section]
\newtheorem{proposition}{Proposition}[section]

\newtheorem*{theorem*}{Theorem}
\newtheorem*{proposition*}{Proposition}

\newtheorem*{remark*}{Remark}

\usepackage{algorithm, algorithmic}

\newcommand{\X}{\mathbf{x}}
\newcommand{\Y}{\mathbf{y}}

\newcommand{\predict}{LIMIT}

\icmltitlerunning{Improving Generalization by Controlling Label-Noise Information in Neural Network Weights}

\begin{document}

\twocolumn[
\icmltitle{Improving Generalization by Controlling \protect\\Label-Noise Information in Neural Network Weights}



\icmlsetsymbol{equal}{*}

\begin{icmlauthorlist}
\icmlauthor{Hrayr Harutyunyan}{isi}
\icmlauthor{Kyle Reing}{isi}
\icmlauthor{Greg Ver Steeg}{isi}
\icmlauthor{Aram Galstyan}{isi}
\end{icmlauthorlist}

\icmlaffiliation{isi}{Information Sciences Institute, University of Southern California, Marina del Rey, CA 90292}

\icmlcorrespondingauthor{Hrayr Harutyunyan}{hrayrh@isi.edu}

\icmlkeywords{Machine Learning, Information Theory, ICML, Overfitting, Generalization}

\vskip 0.3in
]



\printAffiliationsAndNotice{}  

\begin{abstract}
In the presence of noisy or incorrect labels, neural networks have the undesirable tendency to memorize information about the noise.
Standard regularization techniques such as dropout, weight decay or data augmentation sometimes help, but do not prevent this behavior.
If one considers neural network weights as random variables that depend on the data and stochasticity of training, the amount of memorized information can be quantified with the Shannon mutual information between weights and the vector of all training labels given inputs, $I(w ; \Y \mid \X)$.
We show that for any training algorithm, low values of this term correspond to reduction in memorization of label-noise and better generalization bounds.
To obtain these low values, we propose training algorithms that employ an auxiliary network that predicts gradients in the final layers of a classifier without accessing labels.
We illustrate the effectiveness of our approach on versions of MNIST, CIFAR-10, and CIFAR-100 corrupted with various noise models, and on a large-scale dataset Clothing1M that has noisy labels.
\end{abstract}

\section{Introduction}\label{sec:intro}
Supervised learning with deep neural networks has shown great success in the last decade.
Despite having millions of parameters, modern neural networks generalize surprisingly well.
However, their training is particularly susceptible to noisy labels, as shown by \citet{zhang2016understanding} in their analysis of generalization error.
In the presence of noisy or incorrect labels, networks start to memorize the training labels, which degrades the generalization performance~\cite{chen2019understanding}.
At the extreme, standard architectures have the capacity to achieve 100\% classification accuracy on training data, even when labels are assigned at random~\cite{zhang2016understanding}.
Furthermore, standard explicit or implicit regularization techniques such as dropout, weight decay or data augmentation do not directly address nor completely prevent label memorization~\cite{zhang2016understanding, arpit2017closer}.

\begin{figure}[!t]
    \centering
    \begin{subfigure}{\columnwidth}
    \includegraphics[width=\textwidth]{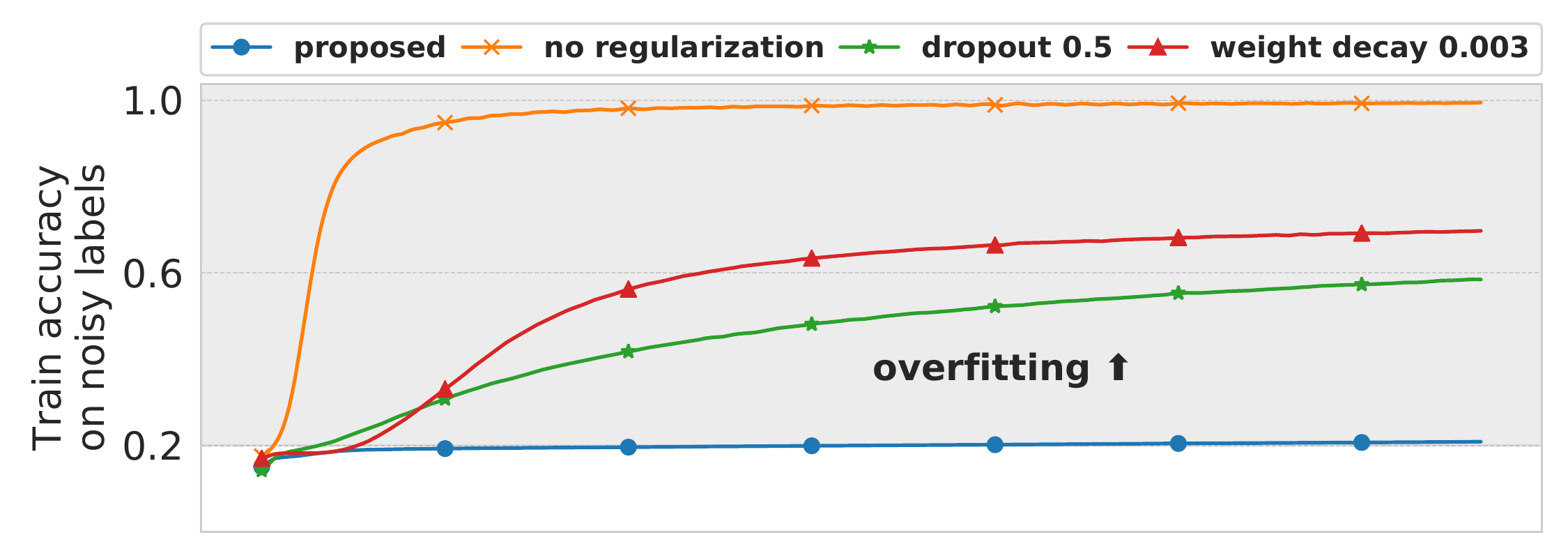}
    \end{subfigure}
    \begin{subfigure}{\columnwidth}
    \includegraphics[width=\textwidth]{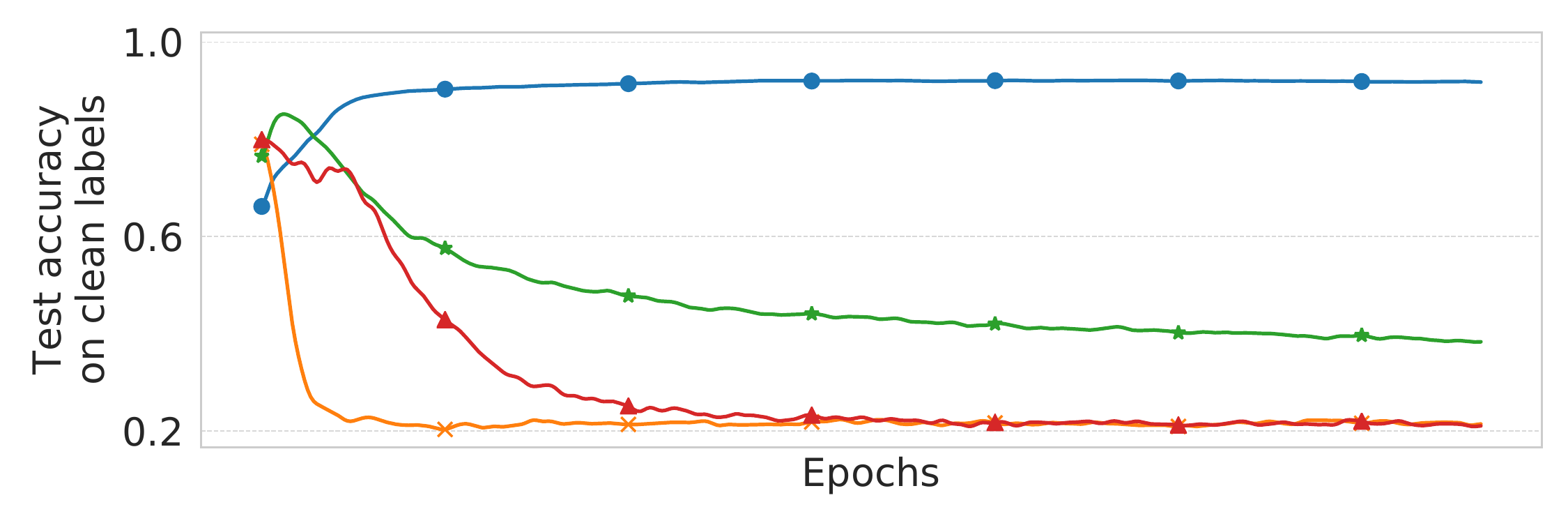}
    \end{subfigure}
    \caption{Neural networks tend to memorize labels when trained with noisy labels (80\% noise in this case), even when dropout or weight decay are applied.
    Our training approach limits label-noise information in neural network weights, avoiding memorization of labels and improving generalization. Please refer to Sec.~\ref{subsec:fano} for more details.}
    \label{fig:preventing_memorization}
\end{figure}
Poor generalization due to label memorization is a significant problem because many large, real-world datasets are imperfectly labeled.
Label noise may be introduced when building datasets from unreliable sources of information or using crowd-sourcing resources like Amazon Mechanical Turk.
A practical solution to the memorization problem is likely to be algorithmic as sanitizing labels in large datasets is costly and time consuming.
Existing approaches for addressing the problem of label-noise and generalization performance include deriving robust loss functions~\cite{natarajan2013learning, mae, gce, dmi}, loss correction techniques~\cite{Sukhbaatar2014TrainingCN,massive,goldberger2016training, patrini2017making}, re-weighting samples~\cite{jiang2017mentornet, ren2018learning}, detecting incorrect samples and relabeling them~\cite{Reed2014TrainingDN, tanaka2018joint, ma2018dimensionality}, and employing two networks that select training examples for each other~\cite{han2018co,Yu2019HowDD}.
We propose an information-theoretic approach that directly addresses the root of the problem. If a classifier is able to correctly predict a training label that is actually random, it must have somehow stored information about this label in the parameters of the model. 
To quantify this information, \citet{emergence} consider weights as a random variable, $w$, that depends on stochasticity in training data and parameter initialization. 
The entire training dataset is considered a random variable consisting of a vector of inputs, $\X$, and a vector of labels for each input, $\Y$.
The amount of label memorization is then given by the Shannon mutual information between weights and labels conditioned on inputs, $I(w ; \Y \mid \X)$. \citet{emergence} show that this term appears in a decomposition of the commonly used expected cross-entropy loss, along with three other individually meaningful terms. Surprisingly, cross-entropy rewards large values of $I(w ; \Y \mid \X)$, which may promote memorization if labels contain information beyond what can be inferred from $\X$.
Such a result highlights that in addition to the network's representational capabilities, the loss function -- or more generally, the learning algorithm -- plays an important role in memorization. To this end, we wish to study the utility of limiting $I(w ; \Y \mid \X)$, and how it can be used to modify training algorithms to reduce memorization.

Our main contributions towards this goal are as follows: 1) We show that low values of $I(w ; \Y \mid \X)$ correspond to reduction in memorization of label-noise, and lead to better generalization gap bounds. 2) We propose training methods that control memorization by regularizing label-noise information in weights.
When the training algorithm is a variant of stochastic gradient descent, one can achieve this by controlling label-noise information in gradients. A promising way of doing this is through an additional network that tries to predict the classifier gradients without using label information. We experiment with two training procedures that incorporate gradient prediction in different ways: one which uses the auxiliary network to penalize the classifier, and another which uses predicted gradients to train it. In both approaches, we employ a regularization that penalizes the L2 norm of predicted gradients to control their capacity.
The latter approach can be viewed as a search over training algorithms, as it implicitly looks for a loss function that balances training performance with label memorization.
3) Finally, we show that the auxiliary network can be used to detect incorrect or misleading labels.
To illustrate the effectiveness of the proposed approaches, we apply them on corrupted versions of MNIST, CIFAR-10, CIFAR-100 with various label noise models, and on the Clothing1M dataset, which already contains noisy labels. We show that methods based on gradient prediction yield drastic improvements over standard training algorithms (like cross-entropy loss), and outperform  competitive approaches designed for learning with noisy labels. 

\section{Label-Noise Information in Weights}\label{sec:label_noise}
We begin by formally introducing a measure of label-noise information in weights, and discuss its connections to memorization and generalization.
Throughout the paper we use several information-theoretic quantities such as entropy: $H(X) = -\mathbb{E}\big[\log p(x)\big]$, mutual information: $I(X ; Y) = H(X) + H(Y) - H(X,Y)$, Kullback–Leibler divergence: $\text{KL}(p(x) || q(x)) = \mathbb{E}_{x\sim p(x)}\left[\log(p(x)/q(x))\right]$ and their conditional variants~\cite{cover}.

Consider a setup in which a labeled dataset, $S = (\X,\Y)$, for data $\X = \left\{x^{(i)}\right\}_{i=1}^n$ and categorical labels $\Y = \left\{y^{(i)}\right\}_{i=1}^n$, is generated from a distribution $p_\theta(x,y)$.
A training algorithm for learning weights $w$ of a fixed probabilistic classifier $f(y \mid x, w)$ can be denoted as a conditional distribution $\mathcal{A}(w \mid S)$.
Given any training algorithm $\mathcal{A}$, its training performance can be measured using the expected cross-entropy:
\begin{equation*}
H_{p,f}(\Y \mid \X, w) = \mathbb{E}_{S}\mathbb{E}_{w|S}\left[\sum_{i=1}^n -\log f (y^{(i)} \mid x^{(i)}, w)\right].
\label{eq:expected-ce}
\end{equation*}
\citet{emergence} present a decomposition of this expected cross-entropy, which reduces to the following when the data generating process is fixed:
\begin{align*}
H_{p,f}(\Y \mid \X, w) = &H(\Y \mid \X) - \overbrace{I(w ; \Y \mid \X)}^\text{memorizing label-noise} \numberthis\label{eq:ce-decomp}\\
&+\mathbb{E}_{\X,w} \text{KL}\big[p(\Y \mid \X) || f(\Y \mid \X, w) \big].
\end{align*}

The problem of minimizing this expected cross-entropy is equivalent to selecting an appropriate training algorithm.
If the labels contain information beyond what can be inferred from inputs (meaning non-zero $H(\Y \mid \X))$, such an algorithm may do well by memorizing the labels through the second term of (\ref{eq:ce-decomp}).
Indeed, minimizing the empirical cross-entropy loss $\mathcal{A}^\text{ERM}(w \mid S) = \delta(w^*)$, where $w^* \in \text{arg}\min{}_{w}\sum_{i=1}^n -\log f (y^{(i)} \mid x^{(i)}, w)$, does exactly that~\cite{zhang2016understanding}.

\subsection{Decreasing $I(w ; \Y \mid \X)$ Reduces Memorization}\label{subsec:fano}
To demonstrate that $I(w ; \Y \mid \X)$ is directly linked to memorization, we prove that any algorithm with small $I(w ; \Y \mid \X)$ overfits less to label-noise in the training set.
\begin{theorem}
Consider a dataset $S=(\X, \Y)$ of $n$ i.i.d. samples,
$\X = \{x^{(i)}\}_{i=1}^n$ and $\Y = \{y^{(i)}\}_{i=1}^n$,
where the domain of labels is a finite set, $\mathcal{Y}$. 
Let $\mathcal{A}(w \mid S)$ be any training algorithm, producing weights for a possibly stochastic classifier $f(y \mid x, w)$.
Let $\widehat{y}^{(i)}$ denote the prediction of the classifier on the $i$-th example and let $e^{(i)} = \mathds{1}\{\widehat{y}^{(i)} \neq y^{(i)}\}$ be a random variable corresponding to predicting $y^{(i)}$ incorrectly.
Then, the following inequality holds:
\begin{equation*}
\mathbb{E}\left[\sum_{i=1}^n e^{(i)}\right] \ge \frac{H(\Y \mid \X) - I(w ; \Y \mid \X) - \sum_{i=1}^n H(e^{(i)})}{ \log \left(\lvert \mathcal{Y} \rvert - 1\right)}.
\end{equation*}
\label{thm:fano}
\end{theorem}
\vspace{-1em}
This result establishes a lower bound on the expected number of prediction errors on the training set, which increases as $I(w ; \Y \mid \X)$ decreases.
For example, consider a corrupted version of the MNIST dataset where each label is changed with probability $0.8$ to a uniformly random incorrect label.
By the above bound, every algorithm for which $I(w ; \Y \mid \X) = 0$ will make at least $80\%$ prediction errors on the training set in expectation.
In contrast, if the weights retain $1$ bit of label-noise information per example, the classifier will make at least 40.5\% errors in expectation.
The proof of Thm.~\ref{thm:fano} uses Fano's inequality and is presented in the supplementary material (Sec.~\ref{subsec:fano_proof}).
Below we discuss the dependence of error probability on $I(w ; \Y \mid \X)$.

\textbf{Remark 1.\quad} If we let $k = |\mathcal{Y}|$ and $r = \frac{1}{n}\mathbb{E}\left[\sum_{i=1}^n e^{(i)}\right]$ denote the expected training error rate, then by Jensen's inequality we can simplify Thm.~\ref{thm:fano} as follows:
\begin{align*}
r &\ge \frac{H(y^{(1)} \mid x^{(1)}) - I(w ; \Y \mid \X) / n - \frac{1}{n}\sum_{i=1}^n H(e^{(i)})}{\log(k - 1)}\\
    &\ge \frac{H(y^{(1)} \mid x^{(1)}) - I(w ; \Y \mid \X) / n - H(r)}{\log(k - 1)}. \numberthis\label{eq:fano_r}
\end{align*}
Solving this inequality for $r$ is challenging. 
One can simplify the right hand side further by bounding $H(e^{(1)}) \le 1$ (assuming that entropies are measured in bits).
However, this will loosen the bound.
Alternatively, we can find the smallest $r_0$ for which (\ref{eq:fano_r}) holds and claim that $r \ge r_0$.

\textbf{Remark 2.\quad} If $|\mathcal{Y}| = 2$, then $\log(|\mathcal{Y}| - 1) = 0$, putting which in (\ref{eq:original-fano}) of supplementary leads to:
\begin{equation*}
H(r) \ge H(y^{(1)} \mid x^{(1)}) - I(w ; \Y \mid \X) / n.
\end{equation*}

\begin{figure}[t]
    \centering
    \includegraphics[width=0.4\textwidth]{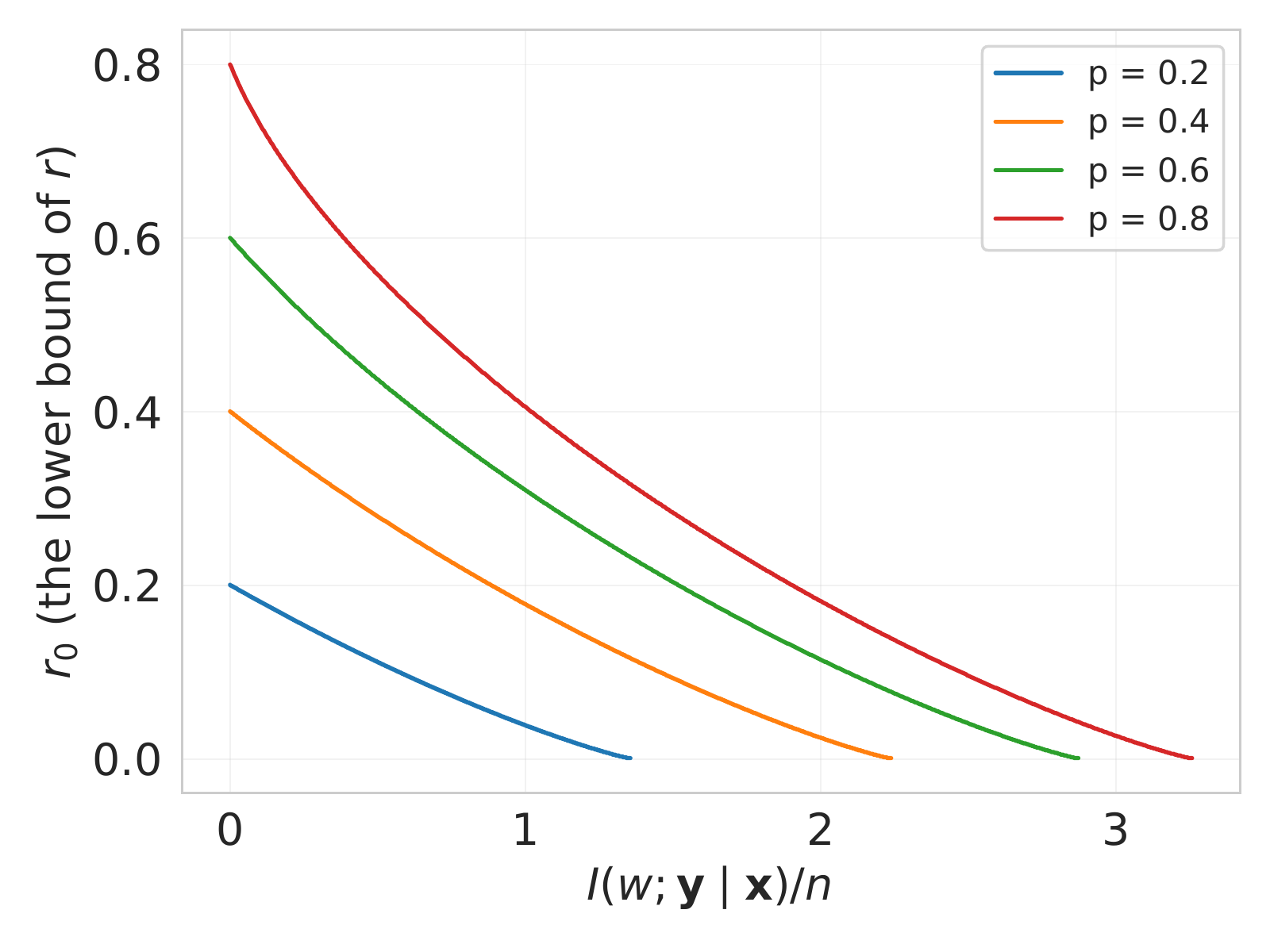}
    \caption{The lower bound $r_0$ on the rate of training errors $r$ Thm.~\ref{thm:fano} establishes for varying values of $I(w ; \Y \mid \X)$, in the case when label noise is uniform and probability of a label being incorrect is $p$.}
    \label{fig:fano}
\end{figure}

\textbf{Remark 3.\quad} When we have uniform label noise where a label is incorrect with probability $p$ ($0 \le p < \frac{k-1}{k}$) and $I(w ; \Y \mid \X) = 0$, the bound of (\ref{eq:fano_r}) is tight, i.e., implies that $r \ge p$. 
To see this, we note that $H(y^{(1)} \mid x^{(1)}) = H(p) + p \log (k-1)$, putting which in (\ref{eq:fano_r}) gives us:
\begin{equation}
    r \ge \frac{H(p) + p \log(k-1) - H(r)}{\log(k-1)} = p + \frac{H(p) - H(r)}{\log(k-1)}.
    \label{eq:fano_uniform}
\end{equation}
Therefore, when $r=p$, the inequality holds, implying that $r_0 \le p$. To show that $r_0 = p$, we need to show that for any $0 \le r < p$, the (\ref{eq:fano_uniform}) does not hold.
Let $r \in [0, p)$ and assume that (\ref{eq:fano_uniform}) holds. Then
\begin{align*}
    r &\ge p + \frac{H(p) - H(r)}{\log(k-1)}\\
      &\ge p + \frac{H(p) - \left(H(p) + (r-p) H'(p)\right)}{\log(k-1)}\\
      &\ge p + \frac{-(r-p) \log(k-1)}{\log(k-1)} = 2p - r.\numberthis\label{eq:fano-zero-mi-case}
\end{align*}
The second line above follows from concavity of $H(x)$; and the third line follows from the fact that $H'(p) > -\log(k-1)$ when $0\le p < (k-1)/k$.
Eq. (\ref{eq:fano-zero-mi-case}) directly contradicts with $r<p$. Therefore, Eq. (\ref{eq:fano_uniform}) cannot hold for any $r < p$.

When $I(w ; \Y \mid \X) > 0$, we can find the smallest $r_0$ by a numerical method.
Fig.~\ref{fig:fano} plots $r_0$ vs $I(w ; \Y \mid \X)$ when the label noise is uniform.
When the label-noise is not uniform, the bound of (\ref{eq:fano_r}) becomes loose as Fano's inequality becomes loose.
We leave the problem of deriving better lower bounds in such cases for a future work.

Thm.~\ref{thm:fano} provides theoretical guarantees that memorization of noisy labels is prevented when $I(w ; \Y \mid \X)$ is small, in contrast to standard regularization techniques -- such as dropout, weight decay, and data augmentation -- which only slow it down~\cite{zhang2016understanding, arpit2017closer}.
To demonstrate this empirically, we compare an algorithm that controls $I(w ; \Y \mid \X)$ (presented in Sec.~\ref{sec:method}) against these regularization techniques on the aforementioned corrupted MNIST setup. We see in Fig.~\ref{fig:preventing_memorization} that explicitly preventing memorization of label-noise information leads to optimal training performance (20\% training accuracy) and good generalization on a non-corrupted validation set. Other approaches quickly exceed 20\% training accuracy by incorporating label-noise information, and generalize poorly as a consequence.
The classifier here is a fully connected neural network with 4 hidden layers each having 512 ReLU units.
The rates of dropout and weight decay were selected according to the performance on a validation set.

\subsection{Decreasing $I(w ; \Y \mid \X)$ Improves Generalization}
The information that weights contain about a training dataset $S$ has previously been linked to generalization~\citep{xu2017information}. The following bound relates the expected difference between train and test performance to the mutual information $I(w ; S)$.

\begin{theorem}{\cite{xu2017information}}
Suppose $\ell(\hat{y},y)$ is a loss function, such that $\ell(f_w(x),y)$ is $\sigma$-sub-Gaussian random variable for each $w$. Let $S=(\X,\Y)$ be the training set, $\mathcal{A}(w\mid S)$ be the training algorithm, and $(\bar{x},\bar{y})$ be a test sample independent from $S$ and $w$. Then the following holds:
\begin{equation}
\small
\left\lvert \mathbb{E}\left[\ell(f_w(\bar{x}),\bar{y}) - \frac{1}{n}\sum_{i=1}^n \ell\left(f_w(x^{(i)}), y^{(i)}\right)\right] \right\rvert \le \sqrt{\frac{2\sigma^2}{n} I(w ; S)}
\label{eq:xu-raginsky}
\end{equation}
\end{theorem}
For good test performance, learning algorithms need to have both a small generalization gap, and good training performance.
The latter may require retaining more information about the training set, meaning there is a natural conflict between increasing training performance and decreasing the generalization gap bound of (\ref{eq:xu-raginsky}).
Furthermore, information in weights can be decomposed as follows: $I(w ; S) = I(w ; \X) + I(w ; \Y \mid \X)$.
We claim that one needs to prioritize reducing $I(w ; \Y \mid \X)$ over $I(w ; \X)$ for the following reason.
When noise is present in the training labels, fitting this noise implies a non-zero value of $I(w ; \Y \mid \X)$, which grows linearly with the number of samples $n$.
In such cases, the generalization gap bound of (\ref{eq:xu-raginsky}) becomes a constant and does not improve as $n$ increases.
To get meaningful generalization bounds via (\ref{eq:xu-raginsky}) one needs to limit $I(w ; \Y \mid \X)$.
We hypothesize that for efficient learning algorithms, this condition might be also sufficient.

\section{Methods Limiting Label Information}\label{sec:method}
We now consider how to design training algorithms that control $I(w ; \Y \mid \X)$.
We assume $f(y \mid x, w) = \text{Multinoulli}(y; s(a))$, with $a$ as the output of a neural network $h_w(x)$, and $s\left(\cdot\right)$ as the softmax function.
We consider the case when $h_w(x)$ is trained with a variant of stochastic gradient descent for $T$ iterations.
The inputs and labels of a mini-batch at iteration $t$ are denoted by $x_t$ and $y_t$ respectively, and are selected using a deterministic procedure (such as cycling through the dataset, or using pseudo-randomness).
Let $w_0$ denote the weights after initialization, and $w_t$ the weights after iteration $t$.
Let $\mathcal{L}(w; x,y)$ be some classification loss function (e.g, cross-entropy loss) and $g^\mathcal{L}_t \triangleq \nabla_{w} \mathcal{L}(w_{t-1}; x_t, y_t)$ be the gradient at iteration $t$.
Let $g_t$ denote the gradients used to update the weights, possibly different from $g^\mathcal{L}_t$.
Let the update rule be $w_t = \Psi(w_0, g_{1:t})$, and $w_T = \Psi(w_0, g_{1:T})$ be the final weights (denoted with $w$ for convenience).

To limit $I(w ; \Y \mid \X)$, the following sections will discuss two approximations which relax the computational difficulty, while still provide meaningful bounds: 1) first, we show that the information in weights can be replaced by information in the gradients; 2) we introduce a variational bound on the information in gradients.
The bound employs an auxiliary network that predicts gradients of the original loss without label information.
We then explore two ways of incorporating predicted gradients: (a) using them in a regularization term for gradients of the original loss, and (b) using them to train the classifier. 

\subsection{Penalizing Information in Gradients}

Looking at (1) it is tempting to add $I(w ; \Y \mid \X)$ as a regularization to the $H_{p,f}(\Y \mid \X, w)$ objective and minimize over all training algorithms:
\begin{equation}
\min_{\mathcal{A}(w \mid D)} H_{p,f}(\Y \mid \X, w) + I(w ; \Y \mid \X).
\label{eq:opt-over-algos}
\end{equation}
This will become equivalent to minimizing $\mathbb{E}_{\X,w} \text{KL}\big[p(\Y \mid \X) || f(\Y \mid \X, w) \big]$.
Unfortunately, the optimization problem of (\ref{eq:opt-over-algos}) is hard to solve for two major reasons.
First, the optimization is over training algorithms (rather than over the weights of a classifier, as in the standard machine learning setup).
Second, the penalty $I(w ; \Y \mid \X)$ is hard to compute/approximate.

To simplify the problem of (\ref{eq:opt-over-algos}), we relate information in weights to information in gradients as follows:
\begin{align*}
I(w ; \Y \mid \X) &\le I(g_{1:T} ; \Y \mid \X)=\sum_{t=1}^T I(g_t ; \Y \mid \X, g_{<t}),\numberthis\label{eq:chain-rule}
\end{align*}
where $g_{1:T}$ and $g_{<t}$ are shorthands for sets  $\{g_1,\ldots,g_T\}$ and $\{g_1,\ldots,g_{t-1}\}$ respectively.
Hereafter, we focus on constraining $I(g_t ; \Y \mid \X, g_{<t})$ at each iteration.
Our task becomes choosing a loss function such that $I(g_t ; \Y \mid \X, g_{<t})$ is small and $f(y \mid x, w_t)$ is a good classifier.
One key observation is that if our task is to minimize label-noise information in gradients it may be helpful to consider gradients with respect to the last layer only and compute the remaining gradients using back-propagation.
As these steps of back-propagation do not use labels, by data processing inequality, subsequent gradients would have at most as much label information as the last layer gradient.

To simplify information-theoretic quantities, we add a small independent Gaussian noise to the gradients of the original loss: $\tilde{g}^\mathcal{L}_t \triangleq g^\mathcal{L}_t + \xi_t$, where $\xi_t \sim \mathcal{N}(0, \sigma_\xi^2 I)$ and $\sigma_\xi$ is small enough to have no significant effect on training (less than $10^{-9}$ is fine).
With this convention, we formulate the following regularized objective function:
\begin{equation}
\min_w \mathcal{L}(w; x_t, y_t) + \lambda I(\tilde{g}^\mathcal{L}_t ; \Y \mid \X, g_{<t}),
\label{eq:penalize-formulation}
\end{equation}
where $\lambda > 0$ is a regularization coefficient.
The term $I(\tilde{g}^\mathcal{L}_t ; \Y \mid \X, g_{<t})$ is a function of $\X$ and $g_{<t}$, or more explicitly, a function $\Phi(w_{t-1}; x_t)$ of $x_t$ and $w_{t-1}$. 
Computing this function would allow the optimization of (\ref{eq:penalize-formulation}) through gradient descent: $g_t = g^\mathcal{L}_t + \xi_t + \nabla_w \Phi(w_{t-1}; x_t)$.
Importantly, label-noise information is equal in both $g_t$ and $\tilde{g}^\mathcal{L}_t$, as the gradient from the regularization is constant given $\X$ and $g_{<t}$:
\begin{align*}
I(g_t ; \Y \mid \X, g_{<t}) &= I(g^\mathcal{L}_t + \xi_t + \nabla_w \Phi(w_{t-1}; x_t) ; \Y \mid \X, g_{<t})\\
&\hspace{-3em}= I(g^\mathcal{L}_t + \xi_t ; \Y \mid \X, g_{<t}) = I(\tilde{g}^\mathcal{L}_t ; \Y \mid \X, g_{<t}).
\end{align*}
Therefore, by minimizing $I(\tilde{g}^\mathcal{L}_t ; \Y \mid \X, g_{<t})$ in (\ref{eq:penalize-formulation}) we minimize $I(g_t ; \Y \mid \X, g_{<t})$, which is used to upper bound $I(w ; \Y \mid \X)$ in (\ref{eq:chain-rule}).
We rewrite this regularization in terms of entropy and discard the constant term, $H(\xi_t)$:
\begin{align*}
I(\tilde{g}^\mathcal{L}_t ; \Y \mid \X, g_{<t}) &= H(\tilde{g}^\mathcal{L}_t \mid \X, g_{<t}) - H(\tilde{g}^\mathcal{L}_t \mid \X, \Y, g_{<t})\\
&= H(\tilde{g}^\mathcal{L}_t \mid \X, g_{<t}) - H(\xi_t)\numberthis\label{eq:entropy-of-grad}.
\end{align*}

\subsection{Variational Bounds on Gradient Information}
The first term in (\ref{eq:entropy-of-grad}) is still challenging to compute, as we typically only have one sample from the unknown distribution $p(y_t \mid x_t)$.
Nevertheless, we can upper bound it with the cross-entropy $H_{p,q} = -\mathbb{E}_{\tilde{g}^\mathcal{L}_t}\left[ \log q_\phi(\tilde{g}^\mathcal{L}_t \mid \X, g_{<t})\right]$, where $q_\phi(\cdot \mid \X, g_{<t})$ is a variational approximation for $p(\tilde{g}^\mathcal{L}_t \mid \X, g_{<t})$:
\begin{align*}
H(\tilde{g}^\mathcal{L}_t \mid \X, g_{<t}) \le -\mathbb{E}\left[\log q_\phi(\tilde{g}^\mathcal{L}_t \mid \X, g_{<t})\right].
\end{align*}
This bound is correct when $\phi$ is a constant or a random variable that depends only on $\X$.
With this upper bound, (\ref{eq:penalize-formulation}) reduces to:
\begin{equation}
    \min_{w, \phi} \mathcal{L}(w ; x_t, y_t) - \lambda \mathbb{E}_{\tilde{g}^\mathcal{L}_t}\left[ \log q_\phi(\tilde{g}^\mathcal{L}_t \mid \X, g_{<t})\right].
    \label{eq:penalize-cross-entropy}
\end{equation}
This formalization introduces a soft constraint on the classifier by attempting to make its gradients predictable without labels $\Y$, effectively reducing $I(g_t ; \Y \mid \X, g_{<t})$.

Assuming $\widehat{y} = s(h_w(x))$ denotes the predicted class probabilities of the classifier and $\mathcal{L}$ is the cross-entropy loss, the gradient with respect to logits $a=h_w(x)$ is $\widehat{y} - y$ (assuming $y$ has a one-hot encoding).
Thus, $I(\tilde{g}^\mathcal{L}_t ; \Y \mid \X, g_{<t}) = I(\widehat{y_t} - y_t + \xi_t ; \Y \mid \X, g_{<t})$ = $I(y_t + \xi_t ; \Y \mid \X, g_{<t})$. Since this expression has no dependence on $w_{t-1}$, it would not serve as a meaningful regularizer. Instead, we descend an additional level to look at gradients of the final layer parameters.
When the final layer of $h_w(x)$ is fully connected with inputs $z$ and weights $U$ (i.e., $a = U z$), the gradients with respect to its parameters is equal to $(\widehat{y} - y)z^T$.
With this gradient, $I(\tilde{g}^\mathcal{L}_t ; \Y \mid \X, g_{<t}) = I((\widehat{y}_t - y_t)z_t^T + \xi_t ; \Y \mid \X, g_{<t}) = I(y_t z_t^T + \xi_t ; \Y \mid \X, g_{<t})$.
There is now dependence on $w_{t-1}$ through $z_t$, as this quantity can be reduced by setting $z_t$ to a small value.
We choose to parametrize $q_\phi(\cdot \mid \X, g_{<t})$ as a Gaussian distribution with mean $\mu_t = (s(a_t) - s(r_\phi(x_t))) z_t^T$ and fixed covariance $\sigma_q I$, where $r_\phi(\cdot)$ is another neural network.
Under this assumption, $H_{p,q}$ becomes proportional to:
\begin{align*}
\mathbb{E}\left[\left\lVert (\widehat{y}_t - y_t) z_t^T + \xi_t - \mu_t\right\rVert^2_2\right] &=  \mathbb{E}\left[\xi_t^2\right]\\
&\hspace{-8em}\quad + \mathbb{E}\left[\left\lVert z_t \right\rVert_2^2 \left\lVert y_t - s(r_\phi(x_t))\right\rVert^2_2\right].
\end{align*}
Ignoring constants and approximating the expectation above with one Monte Carlo sample computed using the label $y_t$, the objective of (\ref{eq:penalize-cross-entropy}) becomes:
\begin{equation}
\min_w \mathcal{L}(w; x_t, y_t) + \lambda \left[\left\lVert z_t \right\rVert_2^2 \left\lVert y_t - s(r_\phi(x_t))\right\rVert^2_2\right].
\label{eq:penalize-final}
\end{equation}
While this may work in principle, in practice the dependence on $w$ is only felt through the norm of $z$, making it too weak to have much effect on the overall objective.
We confirm this experimentally in Sec.~\ref{sec:experiments}.
To introduce more complex dependencies on $w$, one would need to model the gradients of deeper layers.

\subsection{Predicting Gradients without Label Information}\label{subsec:predict}
An alternative approach is to use gradients predicted by $q_\phi(\cdot \mid \X, g_{<t})$ to update classifier weights, i.e., sample $g_t \sim q_\phi(\cdot \mid \X, g_{<t})$.
This is a much stricter condition, as it implies $I(g_t ; \Y \mid \X, g_{<t})=0$ (again assuming $\phi$ is a constant or a random variable that depends only on $\X$).
Note that minimizing $H_{p,q}$ makes the predicted gradient $g_t$ a good substitute for the cross-entropy gradients $\tilde{g}^\mathcal{L}_t$.
Therefore, we write down the following objective function:
\begin{equation}
    \min_{w, \phi} \tilde{\mathcal{L}}(w_{t-1}; \phi, x_t) - \lambda \mathbb{E}_{\tilde{g}^\mathcal{L}_t}\left[ \log q_\phi(\tilde{g}^\mathcal{L}_t \mid \X, g_{<t})\right],
    \label{eq:predict-grad}
\end{equation}
where $\tilde{\mathcal{L}}(w; \phi, x)$ is a probabilistic function defined implicitly such that $\nabla_w \tilde{\mathcal{L}}(w; \phi, x) \sim q_\phi(\cdot \mid x, w)$.
We found that this approach performs significantly better than the penalizing approach of (\ref{eq:penalize-cross-entropy}).
To update $w$ only with predicted gradients, we disable the dependence of the second term of (\ref{eq:predict-grad}) on $w$ in the implementation.
Additionally, one can set $\lambda=1$ above as the first term depends only on $w$, while the second term depends only on $\phi$.

We choose to predict the gradients with respect to logits only and compute the remaining gradients using backpropagation.
We consider two distinct parameterizations for $q_\phi$ -- \textbf{Gaussian:} $q_\phi(\cdot \mid x, w) = \mathcal{N}\left(\mu, \sigma_q^2 I\right)$, and \textbf{Laplace:} $q_\phi(\cdot \mid x, w) = \prod_j \text{Lap}\left(\mu_j, \sigma_q /\sqrt{2})\right)$,
with $\mu = s(a) - s(r_\phi(x))$ and $r_\phi(\cdot)$ being an auxiliary neural network as before.
Under these Gaussian and Laplace parameterizations, $H_{p,q}$ becomes proportional to  $\mathbb{E}||\mu_t - \tilde{g}^\mathcal{L}_t||_2^2$ and $\mathbb{E}||\mu_t - \tilde{g}^\mathcal{L}_t||_1$ respectively.
In the Gaussian case $\phi$ is updated with a mean square error loss (MSE) function, while in the Laplace case it is updated with a mean absolute error loss (MAE).
The former is expected to be faster at learning, but less robust to noise~\cite{mae}.

\subsection{Reducing Overfitting in Gradient Prediction}
In both approaches of (\ref{eq:penalize-cross-entropy}) and (\ref{eq:predict-grad}), the classifier can still overfit if $q_\phi(\cdot \mid x, w)$ overfits.
There are multiple ways to prevent this.
One can choose $q_\phi$ to be parametrized with a small network, or pre-train and freeze some of its layers in an unsupervised fashion.
In this work, we choose to control the L2 norm of the mean of predicted gradients, $\lVert \mu \rVert^2_2$, while keeping the variance $\sigma_q^2$ fixed.
This can be viewed as limiting the capacity of gradients $g_t$.
\begin{proposition}
If $g_t = \mu_t + \epsilon_t$, where $\epsilon_t \sim \mathcal{N}(0, \sigma_q^2 I_d)$ is independent noise, and $\mathbb{E}\left[ \mu^T_t \mu_t \right] \le L^2$, then the following inequality holds:
\begin{align*}
I(g_t ; \Y \mid \X, g_{<t}) \le \frac{d}{2}\log\left(1 + \frac{L^2}{d\sigma_q^2}\right).
\end{align*}
\label{prop:gradient_capacity}
\end{proposition}
\vskip -1em
The proof is provided in the supplementary section~\ref{subsec:grad_info_proof}. The same bound holds when $\epsilon_t$ is sampled from a product of $d$ univariate zero-mean Laplace distributions with variance $\sigma_q^2$, since the proof relies only on $\epsilon_t$ being zero-mean and having variance $\sigma^2_q$.
The final objective of our main method becomes:
\begin{equation*}
    \min_{w, \phi} \tilde{\mathcal{L}}(w; \phi, x_t) - \lambda \mathbb{E}_{\tilde{g}^\mathcal{L}_t}\left[ \log q_\phi(\tilde{g}^\mathcal{L}_t \mid \X, g_{<t})\right] + \beta \lVert \mu_t \rVert_2^2.
\end{equation*}
As before, to update $w$ only with predicted gradients, we disable the dependence of the second and third terms above on $w$ in the implementation.
We name this final approach \textbf{LIMIT} -- \underline{\textbf{l}}imiting label \underline{\textbf{i}}nformation \underline{\textbf{m}}emorization \underline{\textbf{i}}n \underline{\textbf{t}}raining.
We denote the variants with Gaussian and Laplace distributions as LIMIT$_\mathcal{G}$ and LIMIT$_\mathcal{L}$ respectively.
The pseudocode of LIMIT is presented in the supplementary material (Alg.~\ref{alg:training-loop}).
Note that in contrast to the previous approach of (\ref{eq:penalize-formulation}), this follows the spirit of (\ref{eq:opt-over-algos}), in the sense that the optimization over $\phi$ can be seen as optimizing over training algorithms; namely, learning a loss function $\mathcal{L}'$ implicitly through gradients.
With this interpretation, the gradient norm penalty can be viewed as a way to smooth the learned loss, which is a good inductive bias and facilitates learning.

\section{Experiments}\label{sec:experiments}
We set up experiments with noisy datasets to see how well the proposed methods perform for different types and amounts of label noise.
The simplest baselines in our comparison are  standard cross-entropy (CE) and mean absolute error (MAE) loss functions.
The next baseline is the forward correction approach (FW) proposed by~\cite{patrini2017making}, where the label-noise transition matrix is estimated and used to correct the loss function.
Finally, we include the recently proposed determinant mutual information (DMI) loss, which is the log-determinant of the confusion matrix between predicted and given labels~\cite{dmi}.
Both FW and DMI baselines require initialization with the best result of the CE baseline.
To avoid small experimental differences, we implement all baselines, closely following the original implementations of FW and DMI.
We train all baselines except DMI using the ADAM optimizer~\cite{adam} with learning rate $\alpha = 10^{-3}$ and $\beta_1 = 0.9$.
As DMI is very sensitive to the learning rate, we tune it by choosing the best from the following grid of values $\{10^{-3},10^{-4},10^{-5},10^{-6}\}$.
For all baselines, model selection is done by choosing the model with highest accuracy on a validation set that follows the noise model of the corresponding train set.
All scores are reported on a clean test set.
Additional experimental details, including the hyperparameter grids, are presented in supplementary section~\ref{sec:experimental-details}.
The implementation of the proposed method and the code for replicating the experiments is available at \url{https://github.com/hrayrhar/limit-label-memorization}.

\subsection{MNIST with Uniform Label Corruption}
To compare the variants of our approach discussed earlier and see which ones work well, we do experiments on the MNIST dataset with corrupted labels.
In this experiment, we use a simple uniform label-noise model, where each label is set to an incorrect value uniformly at random with probability $p$.
In our experiments we try 4 values of $p$ -- 0\%, 50\%, 80\%, 89\%.
We split the 60K images of MNIST into training and validation sets, containing 48K and 12K samples respectively.
For each noise amount we try 3 different training set sizes -- $10^3$, $10^4$, and $4.8\cdot 10^4$.
All classifiers and auxiliary networks are 4-layer CNNs, with a shared architecture presented in the supplementary (Sec.~\ref{sec:experimental-details}).
For this experiment we include two additional baselines where additive noise (Gaussian or Laplace) is added to the gradients with respect to logits.
We denote these baselines with names ``CE + GN'' and ``CE + LN''.
The comparison with these two baselines demonstrates that the proposed method does more than simply reduce information in gradients via noise.
We also consider a variant of LIMIT where instead of sampling $g_t$ from $q$ we use the predicted mean $\mu_t$.

\begin{table}[!t]
\begin{center}
\small
\begin{tabular}{lcccccccccc}
\toprule
\multirow{2}{*}{Method} & \multicolumn{2}{c}{$p=0.0$} & \multicolumn{2}{c}{$p=0.5$} & \multicolumn{2}{c}{$p=0.8$}\\
\cmidrule(lr){2-3}
\cmidrule(lr){4-5}
\cmidrule(lr){6-7}
& $10^3$ & All & $10^3$ & All & $10^3$ & All\\
\midrule
CE                             & 94.3 & \textbf{99.2} & 71.8 & 97.2 & 27.0 & 87.2\\
CE + GN         & 89.5 & 97.1 & 70.5 & 97.4 & 25.9 & 85.3\\
CE + LN          & 90.0 & 96.7 & 66.8 & 97.6 & 30.2 & 74.5\\
MAE                            & 94.6 & \textbf{99.1} & 75.6 & 98.1 & 25.1 & 93.2\\
FW                             & 93.6 & \textbf{99.2} & 64.3 & 97.3 & 19.0 & 89.1\\
DMI                            & 94.5 & \textbf{99.2} & 79.8 & 98.3 & 30.3 & 88.8\\
Soft reg. (\ref{eq:penalize-final})                     & \textbf{95.7} & \textbf{99.2} & 76.4 & 98.2 & 28.8 & 89.3\\
\predict$_\mathcal{G}$ + S        & \textbf{95.6} & \textbf{99.3} & 82.8 & 98.7 & \textbf{35.9} & 93.4\\
\predict$_\mathcal{L}$ + S        & 94.8 & \textbf{99.3} & \textbf{88.7} & \textbf{98.9} & \textbf{35.6} & \textbf{97.6}\\
\predict$_\mathcal{G}$ - S       & \textbf{95.7} & \textbf{99.3} & 83.3 & 98.6 & \textbf{37.1} & 94.7\\
\predict$_\mathcal{L}$ - S       & 95.0 & \textbf{99.3} & \textbf{88.2} & \textbf{99.0} & \textbf{35.9} & \textbf{97.7}\\
\bottomrule
\end{tabular}
\end{center}
\caption{Test accuracy comparison on multiple versions of MNIST corrupted with uniform label noise. Error bars are presented in the supplementary material (Sec.~\ref{sec:more-results})}
\label{tab:mnist}
\end{table}
Table~\ref{tab:mnist} shows the test performances of different approaches averaged over 5 training/validation splits.
Standard deviations and additional combinations of $p$ and $n$ are presented in the supplementary (See tables~\ref{tab:mnist-with-error-bars-1} and \ref{tab:mnist-with-error-bars-2}).
Additionally, Fig.~\ref{fig:mnist-error-noise80-curves} shows the training and testing performances of the best methods during the training when $p=0.8$ and all training samples are used.
Overall, variants of LIMIT produce the best results and improve significantly over standard approaches.
The variants with a Laplace distribution perform better than those with a Gaussian distribution.
This is likely due to the robustness of MAE.
Interestingly, LIMIT works well and trains faster when the sampling of $g_t$ in $q$ is disabled (rows with ``-S'').
Thus, hereafter we consider this as our primary approach.
As expected, the soft regularization approach of (\ref{eq:penalize-cross-entropy}) and cross-entropy variants with noisy gradients perform significantly worse than LIMIT.
We exclude these baselines in our future experiments.
Additionally, we tested the importance of penalizing norm of predicted gradients by comparing training and testing performances of LIMIT with varying regularization strength $\beta$ in the supplementary (Fig.~\ref{fig:mnist-error-noise80-beta}).
We found that this penalty is essential for preventing memorization.

\begin{figure}[t]
    \centering
    \begin{subfigure}{0.99\columnwidth}
    \includegraphics[width=\textwidth]{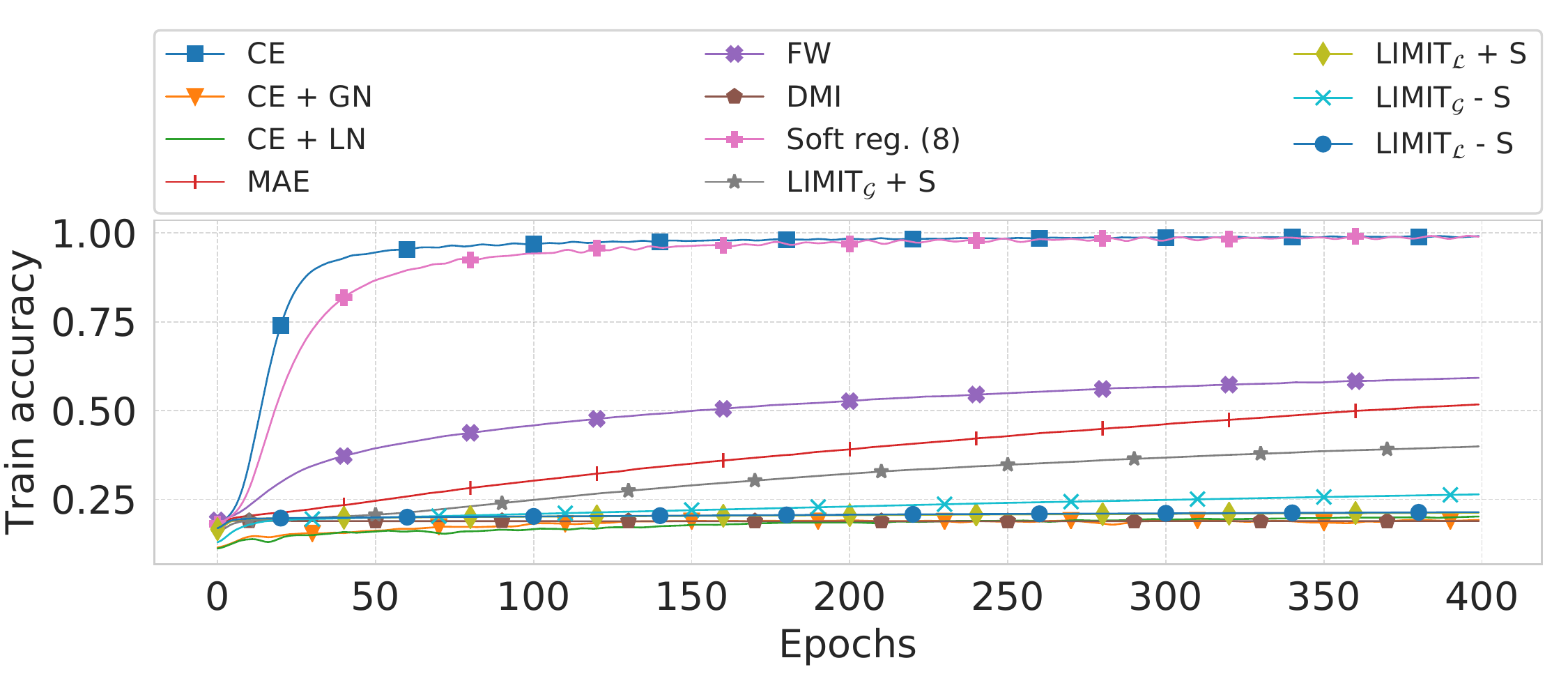}
    \end{subfigure}%
    
    \begin{subfigure}{0.99\columnwidth}
    \includegraphics[width=\textwidth]{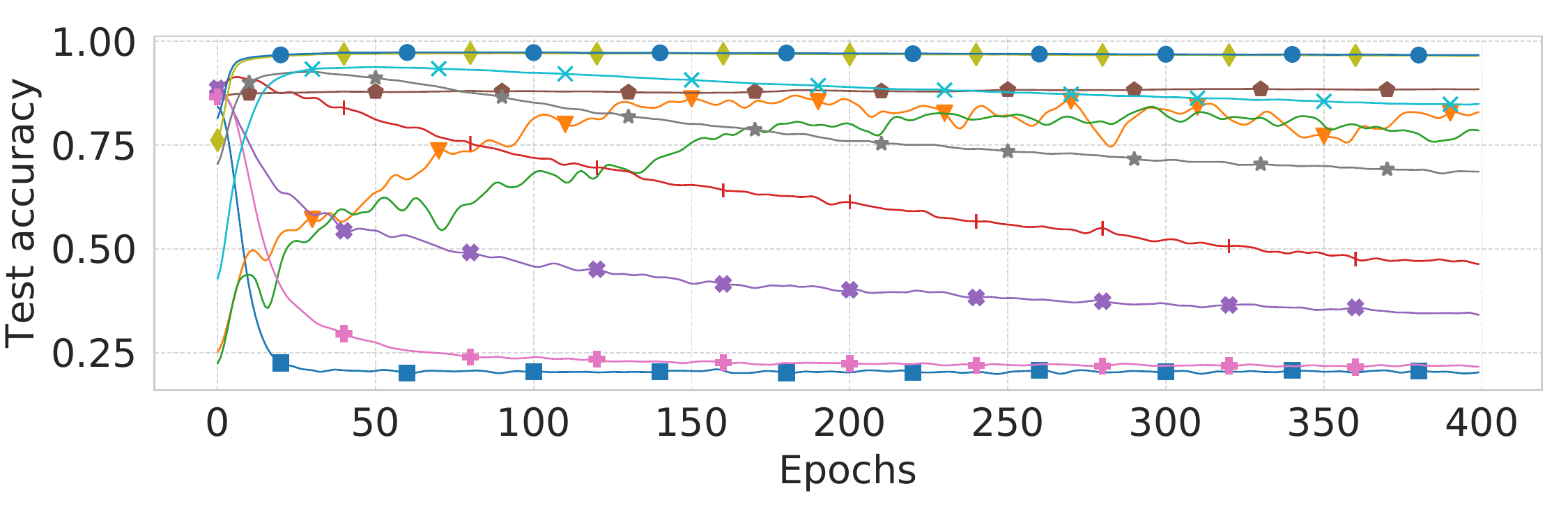}
    \end{subfigure}
    \caption{Smoothed training and testing accuracy plots of various approaches on MNIST with 80\% uniform noise.}
    \label{fig:mnist-error-noise80-curves}
\end{figure}

\begin{figure*}[t]
    \centering
    \begin{subfigure}{0.48\textwidth}
    \includegraphics[width=\textwidth]{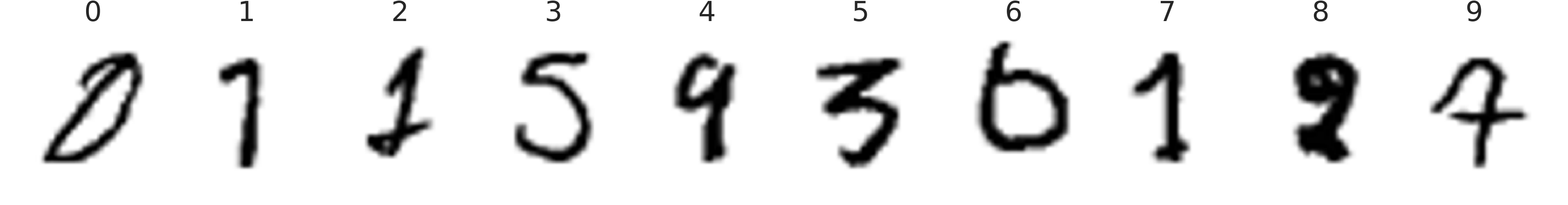}
    \end{subfigure}%
    ~
    \begin{subfigure}{0.48\textwidth}
    \includegraphics[width=\textwidth]{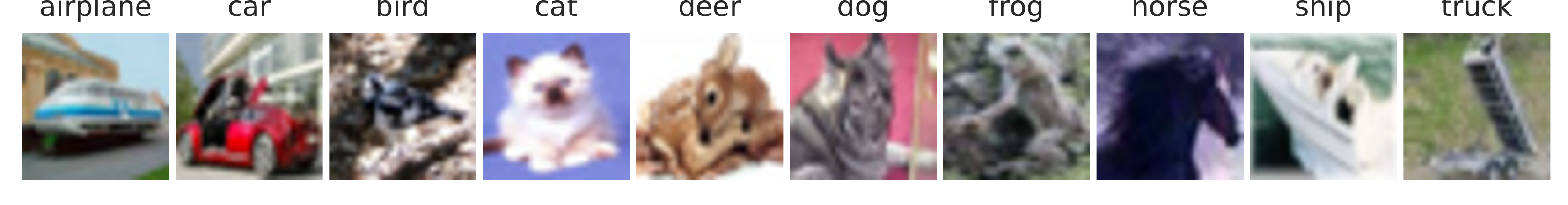}
    \end{subfigure}
    \begin{subfigure}{0.99\textwidth}
    \includegraphics[width=0.99\textwidth]{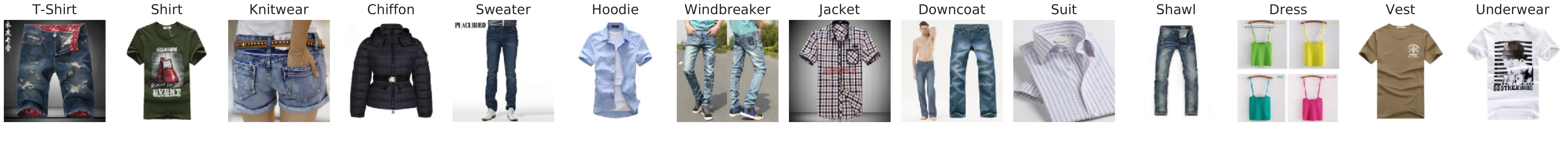}
    \end{subfigure}%
    \caption{Most mislabeled examples in MNIST (top left), CIFAR-10 (top right), and Clothing1M (bottom) datasets, according to the distance between predicted and cross-entropy gradients. More examples are presented in the supplementary (Sec.~\ref{sec:more-results}).}
    \label{fig:confusing-samples}
\end{figure*}

In our approach, the auxiliary network $q$ should not be able to distinguish correct and incorrect samples, unless it overfits.
We found that it learns to predict ``correct'' gradients on examples with incorrect labels (Sec.~\ref{sec:more-results}).
Motivated by this, we use the distance between predicted and cross-entropy gradients to detect samples with incorrect or misleading labels (Fig.~\ref{fig:confusing-samples}).
Additionally, when considering the distance as a score for for classifying correctness of a label,  distance, we get 99.87\% ROC AUC score (Sec.~\ref{sec:more-results}).

\subsection{CIFAR with Uniform and Pair Noise}\label{subsec:cifar-pair}
\begin{table*}[!t]
    \small
    \centering
    \begin{tabular}{lccccccccccc}
    \toprule
    \multirow{3}{*}{Method} & \multicolumn{9}{c}{CIFAR-10} & CIFAR-100 & Clothing1M\\
    \cmidrule(lr){2-10}
    \cmidrule(lr){11-11}
    \cmidrule(lr){12-12}
    & \multicolumn{1}{c}{no noise} & \multicolumn{4}{c}{uniform noise} & \multicolumn{4}{c}{pair noise} & \multicolumn{1}{c}{uniform noise} & \multicolumn{1}{c}{noisy train set}\\
    \cmidrule(lr){2-2}
    \cmidrule(lr){3-6}
    \cmidrule(lr){7-10}
    \cmidrule(lr){11-11}
    \cmidrule(lr){12-12}
     & $0.0$ & $0.2$ & $0.4$  & $0.6$  & $0.8$ & $0.1$ & $0.2$  & $0.3$  & $0.4$ & $0.4$ & -\\
    \midrule
    CE                             & 92.7 & 85.2 & 81.0 & 69.0 & 38.8 & 90.0 & 88.1 & 87.2 & 81.8 & 44.9 $\pm$ 0.5 &  68.91 $\pm$ 0.46\\
    MAE                            & 84.4 & 85.4 & 64.6 & 15.4 & 12.0 & 88.6 & 83.2 & 72.1 & 61.1 &  1.8 $\pm$ 0.1 &  6.52 $\pm$ 0.23\\
    FW                             & 92.9 & 86.2 & 81.4 & 69.7 & 34.4 & 90.1 & 88.0 & 86.8 & 84.6 & 23.3 $\pm$ 0.4 & 68.70 $\pm$ 0.45\\
    DMI                            & 93.0 & 88.3 & 85.0 & 72.5 & 38.9 & 91.4 & 90.6 & 90.4 & 89.6 & 46.1 $\pm$ 0.5 & \textbf{71.19} $\pm$ \textbf{0.43}\\
  \predict$_\mathcal{G}$               & \textbf{93.5} & 90.7 & 86.6 & 73.7 & 38.7 & 92.8 & 91.3 & 89.2 & 86.0 & 58.4 $\pm$ 0.4 & 70.32 $\pm$ 0.42\\
    \predict$_\mathcal{L}$                & 93.1 & 91.5 & 88.2 & 75.7 & 35.8& 91.9 & 91.1 & 88.8 & 84.2 & 49.5 $\pm$ 0.5 & 70.35 $\pm$ 0.45\\
    \predict$_\mathcal{G}$ + init.        & \textbf{93.3} & \textbf{92.4} & \textbf{90.3} & 81.9 & \textbf{44.1} & \textbf{93.3} & \textbf{92.9} & \textbf{90.8} & 88.3 & 59.2 $\pm$ 0.5 & \textbf{71.39} $\pm$ \textbf{0.44}\\
    \predict$_\mathcal{L}$ + init.         & \textbf{93.3} & \textbf{92.2} & \textbf{90.2} & \textbf{82.9} & \textbf{44.3} & \textbf{93.0} & 92.3 & \textbf{91.1} & \textbf{90.0} & \textbf{60.8 $\pm$ 0.5} & 70.53 $\pm$ 0.44\\
    \bottomrule
    \end{tabular}
    \caption{Test accuracy comparison on CIFAR-10, corrupted with various label noise types, on CIFAR-100 with 40\% uniform label noise and on Clothing1M dataset. The error bars are computed by bootstrapping the test set 1000 times. The missing error bars are presented in the supplementary material (Sec.~\ref{sec:more-results}).}
    \label{tab:joint-table}
\end{table*}

Next we consider a harder dataset, CIFAR-10~\cite{krizhevsky2009learning},
with two label noise models: uniform noise and pair noise.
For pair noise, certain classes are confused with some other similar class.
Following the setup of \citet{dmi} we use the following four pairs: truck $\rightarrow$ automobile, bird $\rightarrow$ airplane, deer $\rightarrow$ horse, cat $\rightarrow$ dog.
Note in this type of noise $H(\Y \mid \X)$ is much smaller than in the case of uniform noise.
We split the 50K images of CIFAR-10 into training and validation sets, containing 40K and 10K samples respectively.
For the CIFAR experiments we use ResNet-34 networks~\cite{he2016deep} with standard data augmentation, consisting of random horizontal flips and random 28x28 crops padded back to 32x32.
For our proposed methods, the auxiliary network $q$ is ResNet-34 as well.
We noticed that for more difficult datasets, it may happen that while $q$ still learns to produce good gradients, the updates with these less informative gradients may corrupt the initialization of the classifier.
For this reason, we add an additional variant of LIMIT, which initializes the $q$ network with the best CE baseline, similar to the DMI and FW baselines.

Table~\ref{tab:joint-table} presents the results on CIFAR-10.
Again, variants of LIMIT improve significantly over standard baselines, especially in the case of uniform label noise.
As expected, when $q$ is initialized with the best CE model (similar to FW and DMI), the results are better.
As in the case of MNIST, our approach helps even when the dataset is noiseless.

\textbf{CIFAR-100.\quad}
To test proposed methods on a classification task with many classes, we apply them on CIFAR-100 with 40\% uniform noise.
We use the same networks as in the case of CIFAR-10.
Results presented in Table~\ref{tab:joint-table} indicate several interesting phenomena.
First, training with the MAE loss fails, which was observed by other works as well~\cite{gce}.
The gradient of MAE with respect to logits is $f(x)_y (\widehat{y}-y)$.
When $f(x)_y$ is small, there is small signal to fix the mistake.
In fact, in the case of CIFAR-100, $f(x)_y$ is approximately 0.01 in the beginning, slowing down the training.
The performance of FW degrades as the approximation error of noise transition matrix become large.
The DMI does not give significant improvement over CE due to numerical issues with computing a determinant of a 100x100 confusion matrix.
LIMIT$_\mathcal{L}$ performs worse than other variants, as training $q$ with MAE becomes challenging.
However, performance improves when $q$ is initialized with the CE model.
LIMIT$_\mathcal{G}$ does not suffer from the mentioned problem and works with or without initialization.

\vspace{-0.35em}
\subsection{Clothing1M}
Finally, as in our last experiment, we consider the Clothing1M dataset~\cite{xiao2015learning}, which has 1M images labeled with one of 14 possible clothing labels.
The dataset has very noisy training labels, with roughly 40\% of examples incorrectly labeled.
More importantly, the label noise in this dataset is realistic and instance dependent.
For this dataset we use ResNet-50 networks and employ standard data augmentation, consisting of random horizontal flips and random crops of size 224x224 after resizing images to size 256x256.
The results shown in the last column of Table ~\ref{tab:joint-table} demonstrate that DMI and \predict{} with initialization perform the best, producing similar results.

\section{Related Work}\label{sec:related}
Our approach is related to many works that study memorization and learning with noisy labels.
Our work also builds on theoretical results studying how generalization relates to information in neural network weights.
In this section we present the related work and discuss the connections.

\subsection{Learning with Noisy Labels}
Learning with noisy labels is a longstanding problem and has been studied extensively~\cite{survey}.
Many works studied and proposed loss functions that are robust to label noise.
\citet{natarajan2013learning} propose robust loss functions for binary classification with label-dependent noise.
\citet{mae} generalize this result for multiclass classification problem and show that the mean absolute error (MAE) loss function is tolerant to label-dependent noise.
\citet{gce} propose a new loss function, called generalized cross-entropy (GCE), that interpolates between MAE and CE with a single parameter $q \in [0,1]$.
\citet{dmi} propose a new loss function (DMI), which is equal to the log-determinant of the confusion matrix between predicted and given labels, and show that it is robust to label-dependent noise.
These loss functions are robust in the sense that the best performing hypothesis on clean data and noisy data are the same in the regime of infinite data.
When training on finite datasets, training with these loss functions may result in memorization of training labels.

Another line of research seeks to estimate label-noise and correct the loss function accordingly~\cite{Sukhbaatar2014TrainingCN,massive,goldberger2016training, patrini2017making,hendrycks2018using,safeguard2019}.
Some works use meta-learning to treat the problem of noisy/incomplete labels as a decision problem in which one determines the reliability of a sample~\cite{jiang2017mentornet, ren2018learning, shu2019meta}.
Others seek to detect incorrect examples and relabel them~\cite{Reed2014TrainingDN, tanaka2018joint, ma2018dimensionality, han2019deep, arazo2019unsupervised}.
\citet{han2018co, Yu2019HowDD} employ an approach where two networks select training examples for each other using the small-loss trick. While our approach also has a teaching component, the network uses all samples instead of filtering.
\citet{li2019learning} propose a meta-learning approach that optimizes a classification loss along with a consistency loss between predictions of a mean teacher and predictions of the model after a single gradient descent step on a synthetically labeled mini-batch.


Some approaches assume particular label-noise models, while our approach assumes that $H(\Y \mid \X) > 0$, which may happen because of any type of label noise or attribute noise (e.g., corrupted images or partially observed inputs).
Additionally, the techniques used to derive our approach can be adopted for regression or multilabel classification tasks.
Furthermore, some methods require access to small clean validation data, which is not required in our approach.

\subsection{Information in Weights and Generalization}
Defining and quantifying information in neural network weights is an open challenge and has been studied by multiple authors.
One approach is to relate information in weights to their description length.
A simple way of measuring description length was proposed by \citet{hinton93MDL} and reduces to the L2 norm of weights.
Another way to measure it is through the intrinsic dimension of an objective landscape~\cite{li2018intrinsic, blier2018description}.
\citet{li2018intrinsic} observed that the description length of neural network weights grows when they are trained with noisy labels~\cite{li2018intrinsic}, indicating memorization of labels.

\citet{emergence} define information in weights as the KL divergence from the posterior of weights to the prior.
In a subsequent study they provide generalization bounds involving the KL divergence term~\cite{achille2019information}.
Similar bounds were derived in the PAC-Bayesian setup and have been shown to be non-vacuous~\cite{nonvacuous}.
With an appropriate selection of prior on weights, the above KL divergence becomes the Shannon mutual information between the weights and training dataset, $I(w ; S)$.
\citet{xu2017information} derive generalization bounds that involve this latter quantity.
\citet{pensia2018generalization} upper bound $I(w ; S)$ when the training algorithm consists of iterative noisy updates.
They use the chain-rule of mutual information as we did in (\ref{eq:chain-rule}) and bound information in updates by adding independent noise.
It has been observed that adding noise to gradients can help to improve generalization in certain cases~\cite{neelakantan2015adding}.
Another approach restricts information in gradients by clipping them~\cite{menon2020can} .

\citet{emergence} also introduce the term $I(w ; \Y \mid \X)$ and show the decomposition of the  cross-entropy described in (\ref{eq:ce-decomp}).
In a recent work, \citet{yin2020metalearning} consider a similar term in the context of meta-learning and use it as a regularization to prevent memorization of meta-testing labels.
Given a meta-learning dataset $\mathcal{M}$, they consider the information in the meta-weights $\theta$ about the labels of meta-testing tasks given the inputs of meta-testing tasks, $I(\theta ; \mathcal{Y}\mid \mathcal{X})$.
They bound this information with a variational upper bound $\text{KL}\left(q(\theta \mid \mathcal{M}) || r(\theta)\right)$ and use multivariate Gaussian distributions for both.
For isotropic Gaussians with equal covariances, the KL divergence reduces to $\lVert \theta - \theta_0 \rVert_2^2$, which was studied by~\citet{hu2020simple} as a regularization to achieve robustness to label-noise.
Note that this bounds not only $I(\theta ; \mathcal{Y} \mid \mathcal{X})$ but also $I(\theta ; \mathcal{X}, \mathcal{Y})$. 
In contrast, we bound only $I(w ; \Y \mid \X)$ and work with information in gradients.

\section{Conclusion and Future Work}
Several recent theoretical works have highlighted the importance of the information about the training data that is memorized in the weights.
We distinguished two components of it and demonstrated that the conditional mutual information of weights and labels given inputs is closely related to memorization of labels and generalization performance.
By bounding this quantity in terms of information in gradients, we were able to derive the first practical schemes for controlling label information in the weights and demonstrated that this outperforms approaches for learning with noisy labels.
In the future, we plan to explore ways of improving the bound of (\ref{eq:chain-rule}) and to design a better bottleneck in the gradients. Additionally, we aim to extend the presented ideas to reducing instance memorization.

\section*{Acknowledgements}
We thank Alessandro Achille for his valuable comments. We also thank the anonymous reviewers whose suggestions helped improve this manuscript. Hrayr Harutyunyan was supported by the USC Annenberg Fellowship. This work is based in part on research sponsored by Air Force Research Laboratory (AFRL) under agreement number FA8750-19-1-1000. The U.S. Government is authorized to reproduce and distribute reprints for Government purposes notwithstanding any copyright notation therein. The views and conclusions contained herein are those of the authors and should not be interpreted as necessarily representing the official policies or endorsements, either expressed or implied, of Air Force Laboratory, DARPA or the U.S. Government.

\bibliography{main}
\bibliographystyle{icml2020}

\clearpage
\onecolumn
\appendix

\renewcommand{\thefigure}{A\arabic{figure}}
\setcounter{figure}{-1}

\begin{center}
\textbf{\large Supplementary material: Improving Generalization by Controlling Label-Noise Information in Neural Network Weights}
\end{center}
\vskip 2em

\begin{algorithm}[h]
\caption{\footnotesize LIMIT: limiting label information memorization in training.\\Our implementation is available at \url{https://github.com/hrayrhar/limit-label-memorization}.}
\begin{algorithmic}
    \STATE {\bfseries Input:} Training dataset $\{(x^{(i)}, y^{(i)})\}_{i=1}^n$.
    \STATE {\bfseries Input:} Gradient norm regularization coefficient $\beta$.
    \COMMENT{$\lambda$ is set to $1$}
    \STATE Initialize classifier $f(y \mid x, w)$ and gradient predictor $q_\phi(\cdot \mid \X, g_{<t})$.
    \FOR{$t=1..T$}
        \STATE Fetch the next batch $(x_t, y_t)$ and compute the predicted logits $a_t$.
        \STATE Compute the cross-entropy gradient, $g^\mathcal{L}_t \leftarrow s(a_t) - y_t$.
        \IF{sampling of gradients is enabled}
            \STATE $g_t \sim q_\phi(\cdot \mid \X, g_{<t})$.
        \ELSE
            \STATE $g_t \leftarrow \mu_t$ \COMMENT{the mean of predicted gradient}
        \ENDIF
        \STATE Starting with $g_t$, backpropagate to compute the gradient with respect to $w$.
        \STATE Update $w_{t-1}$ to $w_t$.
        \STATE Update $\phi$ using the gradient of the following loss: $-\log q_\phi(\tilde{g}^\mathcal{L}_t \mid \X, g_{<t}) + \beta \lVert \mu_t \rVert_2^2$.
    \ENDFOR
\end{algorithmic}
\label{alg:training-loop}
\end{algorithm}

\section{Proofs}
This section presents the proofs and some remarks that were not included in the main text due to space constraints.

\subsection{Proof of Thm.~\ref{thm:fano}}\label{subsec:fano_proof}
\begin{theorem}{\textbf{(Thm.~\ref{thm:fano} restated)}}
Consider a dataset $S=(\X, \Y)$ of $n$ i.i.d. samples,
$\X = \{x^{(i)}\}_{i=1}^n$ and $\Y = \{y^{(i)}\}_{i=1}^n$,
where the domain of labels is a finite set, $\mathcal{Y}$, with $|\mathcal{Y}| > 2$.
Let $\mathcal{A}(w \mid S)$ be any training algorithm, producing weights for possibly stochastic classifier $f(y \mid x, w)$.
Let $\widehat{y}^{(i)}$ denote the prediction of the classifier on $i$-th example and $e^{(i)} = \mathds{1}\{\widehat{y}^{(i)} \neq y^{(i)}\}$ be a random variable corresponding to predicting $y^{(i)}$ incorrectly.
Then, the following holds
\begin{equation*}
\mathbb{E}\left[\sum_{i=1}^n e^{(i)}\right] \ge \frac{H(\Y \mid \X) - I(w ; \Y \mid \X) - \sum_{i=1}^n H(e^{(i)})}{ \log \left(\lvert \mathcal{Y} \rvert - 1\right)}.
\end{equation*}
\label{thm:fano_repeated}
\end{theorem}
\begin{proof}
For each example we consider the following Markov chain:
\begin{equation*}
y^{(i)} \rightarrow \left[\begin{matrix} \X \\ \Y \end{matrix}\right] \rightarrow \left[\begin{matrix} x^{(i)} \\ w \end{matrix}\right] \rightarrow \widehat{y}^{(i)}.
\end{equation*}
In this setup Fano's inequality gives a lower bound for the error probability:
\begin{equation}
H(e^{(i)}) + \mathbb{P}(e^{(i)} = 1) \log \left(\lvert \mathcal{Y} \rvert - 1\right) \ge H(y^{(i)} \mid x^{(i)}, w),
\label{eq:original-fano}
\end{equation}
which can be written as:
\begin{equation*}
\mathbb{P}(e^{(i)} = 1) \ge \frac{H(y^{(i)} \mid x^{(i)}, w) - H(e^{(i)})}{ \log \left(\lvert \mathcal{Y} \rvert - 1\right)}.
\end{equation*}
Summing this inequality for $i=1,\ldots,n$ we get
\begin{align*}
\sum_{i=1}^n \mathbb{P}(e^{(i)} = 1) &\ge \frac{\sum_{i=1}^n\left(H(y^{(i)} \mid x^{(i)}, w) - H(e^{(i)})\right)}{ \log \left(\lvert \mathcal{Y} \rvert - 1\right)}\\
&\ge \frac{\sum_{i=1}^n\left(H(y^{(i)} \mid \X, w) - H(e^{(i)})\right)}{ \log \left(\lvert \mathcal{Y} \rvert - 1\right)}\\
&\ge \frac{H(\Y \mid \X, w) - \sum_{i=1}^n H(e^{(i)})}{ \log \left(\lvert \mathcal{Y} \rvert - 1\right)}.
\end{align*}
The correctness of the last step follows from the fact that total correlation is always non-negative~\cite{cover}:
\begin{equation*}
\sum_{i=1}^n H(y^{(i)} \mid \X, w) - H(\Y \mid \X, w) = \text{TC}(\Y \mid \X, w) \ge 0.
\end{equation*}
Finally, using the fact that $H(\Y \mid \X, w) = H(\Y \mid \X) - I(w ; \Y \mid \X)$, we get that the desired result:
\begin{equation}
\mathbb{E}\left[\sum_{i=1}^n e^{(i)}\right] \ge \frac{H(\Y \mid \X) - I(w ; \Y \mid \X) - \sum_{i=1}^n H(e^{(i)})}{ \log \left(\lvert \mathcal{Y} \rvert - 1\right)}.
\label{eq:lower-bound-num-errors}
\end{equation}
\end{proof}

\subsection{Proof of Prop.~\ref{prop:gradient_capacity}}\label{subsec:grad_info_proof}
\begin{proposition}{\textbf{(Prop.~\ref{prop:gradient_capacity} restated)}}
If $g_t = \mu_t + \epsilon_t$, where $\epsilon_t \sim \mathcal{N}(0, \sigma_q^2 I_d)$ is an independent noise and $\mathbb{E}\left[ \mu^T_t \mu_t \right] \le L^2$, then the following inequality holds:
\begin{align*}
I(g_t ; \Y \mid \X, g_{<t})) \le \frac{d}{2}\log\left(1 + \frac{L^2}{d\sigma_q^2}\right).
\end{align*}
\label{prop:gradient_capacity_restated}
\end{proposition}

\begin{proof}
Given that $\epsilon_t$ and $\mu_t$ are independent, let us bound the expected L2 norm of $g_t$:
\begin{align*}
\mathbb{E}\left[ g_t^T g_t \right] &= \mathbb{E}\left[ (\epsilon_t + \mu_t)^T (\epsilon_t + \mu_t) \right]\\
&=\mathbb{E}\left[\epsilon_t^T \epsilon_t\right] + \mathbb{E}\left[\mu_t^T\mu_t\right]\\
&\le d \sigma_q^2  + L^2.
\end{align*}
Among all random variables $Z$ with $\mathbb{E}[Z^T Z] \le C$ the Gaussian distribution $Y \sim \mathcal{N}\left(0, \frac{C}{d} I_d\right)$ has the largest entropy, given by $H(Y) = \frac{d}{2}\log\left(\frac{2\pi e C}{d}\right)$.
Therefore,
\begin{align*}
H(g_t) \le \frac{d}{2}\log\left(\frac{2\pi e (d \sigma_q^2 + L^2)}{d}\right).
\end{align*}
With this we can upper bound the $I(g_t ; \Y \mid \X, g_{<t})$ as follows:
\begin{align*}
I(g_t ; \Y \mid \X, g_{<t}) &= H(g_t \mid \X, g_{<t}) - H(g_t \mid \X, \Y, g_{<t})\\
&= H(g_t \mid \X, g_{<t}) - H(\epsilon_t)\\
&\le \frac{d}{2}\log\left(\frac{2\pi e (d \sigma_q^2 + L^2)}{d}\right)\numberthis\label{eq:max-ent} - \frac{d}{2}\log\left(2\pi e \sigma_q^2\right)\\
&= \frac{d}{2}\log\left(1 + \frac{L^2}{d\sigma_q^2}\right).
\end{align*}
\end{proof}
Note that the proof will work for arbitrary $\epsilon_t$ that has zero mean and independent components, where the L2 norm of each component is bounded by $\sigma_q^2$.
This holds because in such cases $H(\epsilon_t) \le \frac{d}{2} \log(2\pi e \sigma^2_q)$ (as Gaussians have highest entropy for fixed L2 norm) and the transition of (\ref{eq:max-ent}) remains correct.
Therefore, the same result holds when $\epsilon_t$ is sampled from a product of univariate zero-mean Laplace distributions with scale parameter $\sigma_q/\sqrt{2}$ (which makes the second moment equal to $\sigma_q^2$).

A similar result has been derived by \citet{pensia2018generalization} (lemma 5) to bound $I(w_t ; (x_t, y_t) \mid w_{t-1})$.

\section{Experimental Details}\label{sec:experimental-details}
In this section we describe the details of experiments and implementations.

\paragraph{Classifier architectures.}
The architecture of classifiers used in MNIST experiments is presented in Table~\ref{tab:mnist-classifier}.
The ResNet-34 used in CIFAR-10 and CIFAR-100 experiments differs from the standard ResNet-34 architecture (which is used for $224 \times 224$ images) in two ways: (a) the first convolutional layer has 3x3 kernels and stride 1 and (b) the max pooling layer after it is skipped.
The architecture of ResNet-50 used in the Clothing1M experiment follows the original~\cite{he2016deep}.

\begin{table}[t]
    \centering
    \small
    \begin{tabular}{ll}
    \toprule
    Layer type & Parameters\\
    \midrule
    Conv & 32 filters, $4 \times 4$ kernels, stride 2, padding 1, batch normalization, ReLU\\
    Conv & 32 filters, $4 \times 4$ kernels, stride 2, padding 1, batch normalization, ReLU\\
    Conv & 64 filters, $3 \times 3$ kernels, stride 2, padding 0, batch normalization, ReLU\\
    Conv & 256 filters, $3 \times 3$ kernels, stride 1, padding 0, batch normalization, ReLU\\
    FC & 128 units, ReLU\\
    FC & 10 units, linear activation\\
    \bottomrule
    \end{tabular}
    \caption{The architecture of MNIST classifiers.}
\label{tab:mnist-classifier}
\end{table}

\paragraph{Hyperparameter search.}
The CE, MAE, and FW baselines have no hyperparameters.
For the DMI, we tuned the learning rate by setting the best value from the following list: $\{10^{-3},10^{-4},10^{-5},10^{-6}\}$.
The soft regularization approach of (\ref{eq:penalize-final}) has two hyperparameters: $\lambda$ and $\beta$.
We select $\lambda$ from $[0.001, 0.01, 0.03, 0.1]$ and $\beta$ from $[0.0, 0.01, 0.1, 1.0, 10.0]$.
The objective of LIMIT instances has two terms: $\lambda H_{p,q}$ and $\beta \lVert \mu_t \rVert_2^2$. 
Consequently, we need only one hyperparameter instead of two. We choose to set $\lambda = 1$ and select $\beta$ from $[0.0, 0.1, 0.3, 1.0, 3.0, 10.0, 30.0, 100.0]$.
When sampling is enabled, we select $\sigma_q$ from $[0.01, 0.03, 0.1, 0.3]$.
In MNIST and CIFAR experiments, we trained all models for 400 epochs and terminated the training early when the best validation accuracy was not improved in the last 100 epochs.
All models for Clothing1M were trained for 30 epochs.

\section{Additional Results}\label{sec:more-results}

\paragraph{Effectiveness of gradient norm penalty.} In the main text we discussed that the proposed approach may overfit if the gradient predictor $q_\phi(\cdot \mid \X, g_{<t})$ overfits and proposed to penalize the L2 norm of predicted gradients as a simply remedy for this issue.
To demonstrate the effectiveness of this regularization, we present the training and testing accuracy curves of LIMIT with varying values of $\beta$ in Fig.~\ref{fig:mnist-error-noise80-beta}.
We see that increasing $\beta$ decreases overfitting on the training set and usually results in better generalization.

\begin{figure*}[t]
    \centering
    \begin{subfigure}{0.49\textwidth}
    \includegraphics[width=\textwidth]{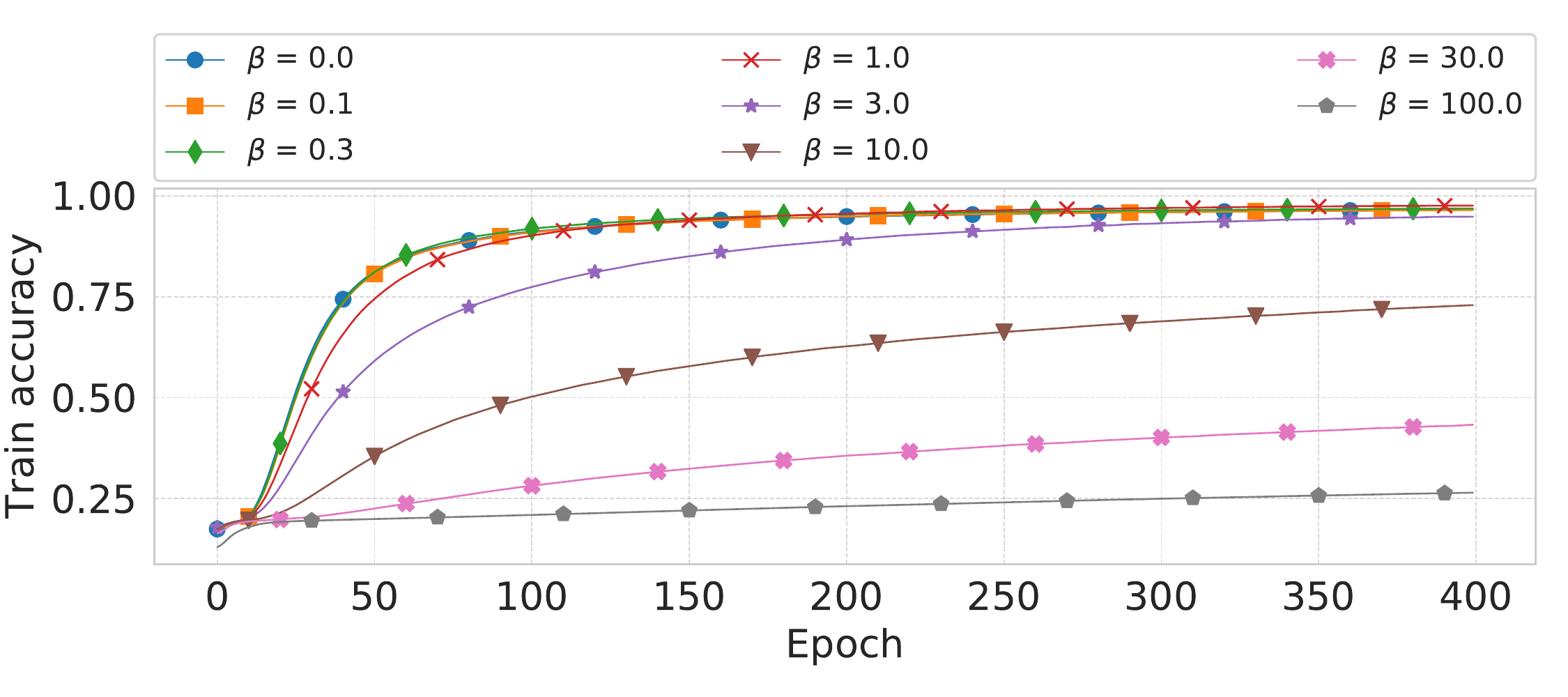}
    \caption{Training performance of LIMIT$_\mathcal{G} - S$}
    \end{subfigure}%
    \begin{subfigure}{0.49\textwidth}
    \includegraphics[width=\textwidth]{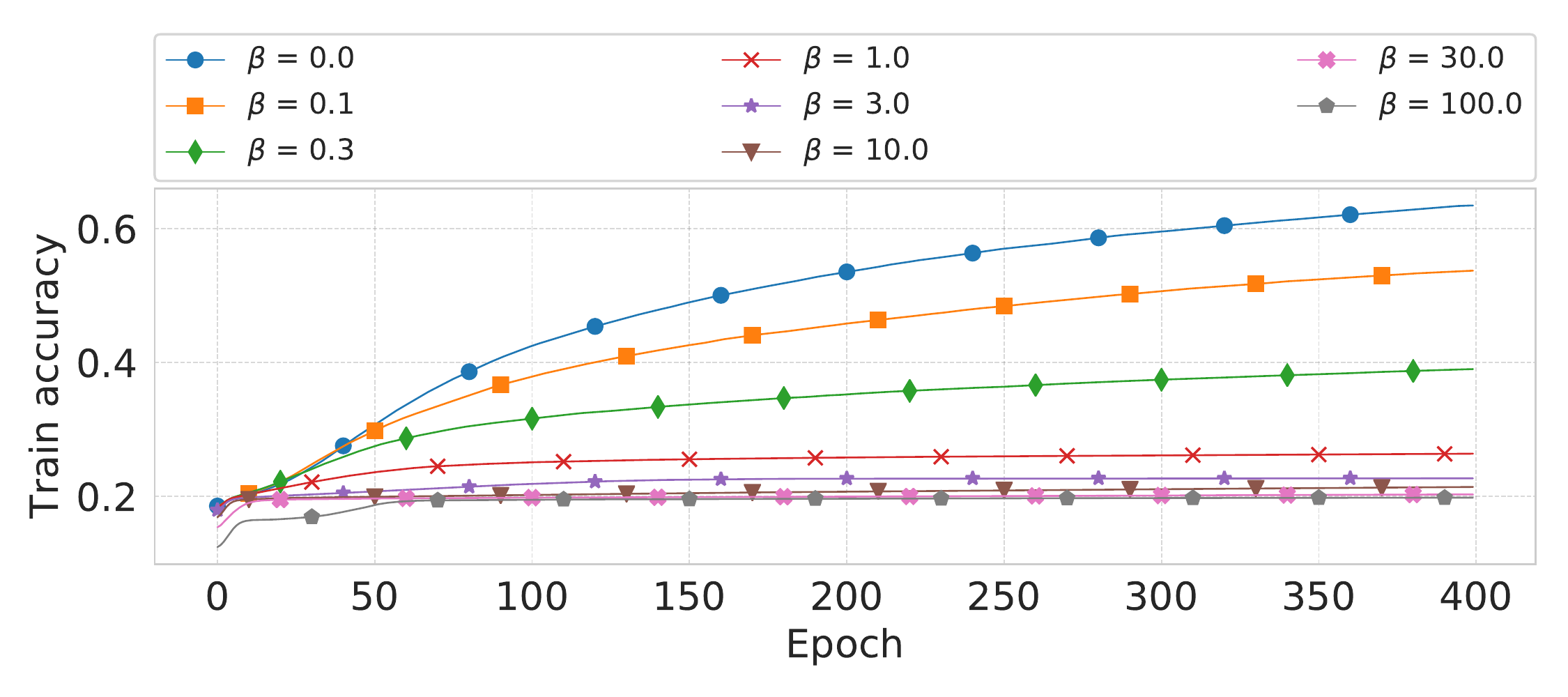}
    \caption{Training performance of LIMIT$_\mathcal{L} - S$}
    \end{subfigure}
    
    \begin{subfigure}{0.49\textwidth}
    \includegraphics[width=\textwidth]{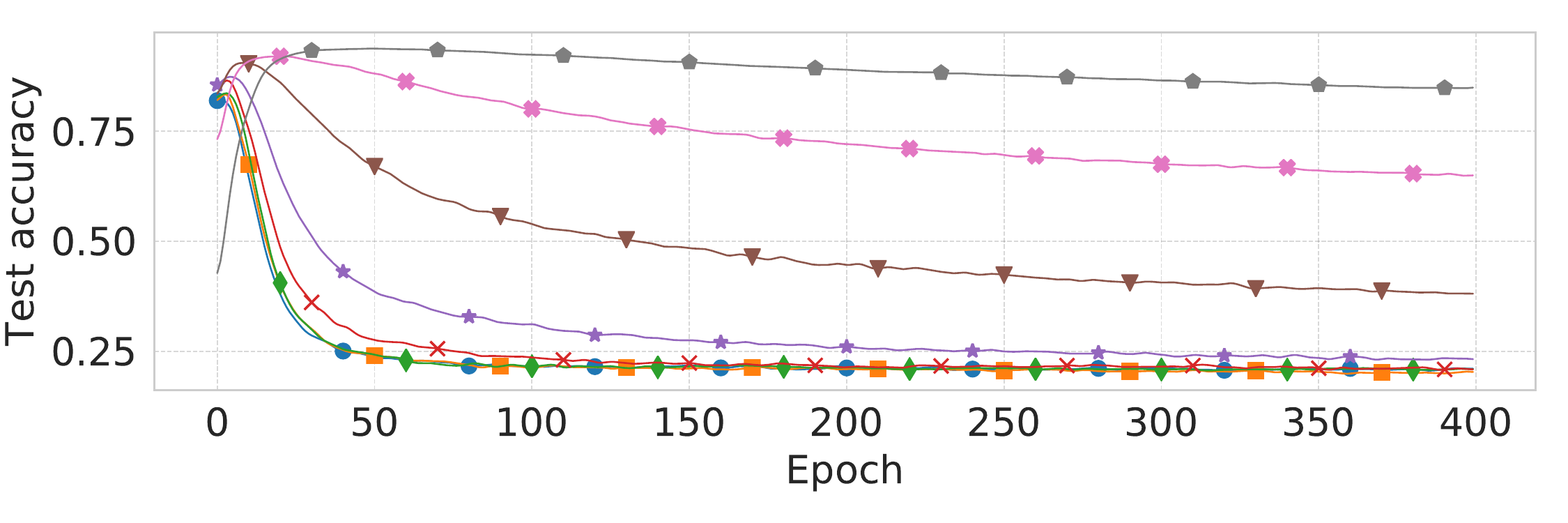}
    \caption{Testing performance of LIMIT$_\mathcal{G} - S$}
    \end{subfigure}%
    \begin{subfigure}{0.49\textwidth}
    \includegraphics[width=\textwidth]{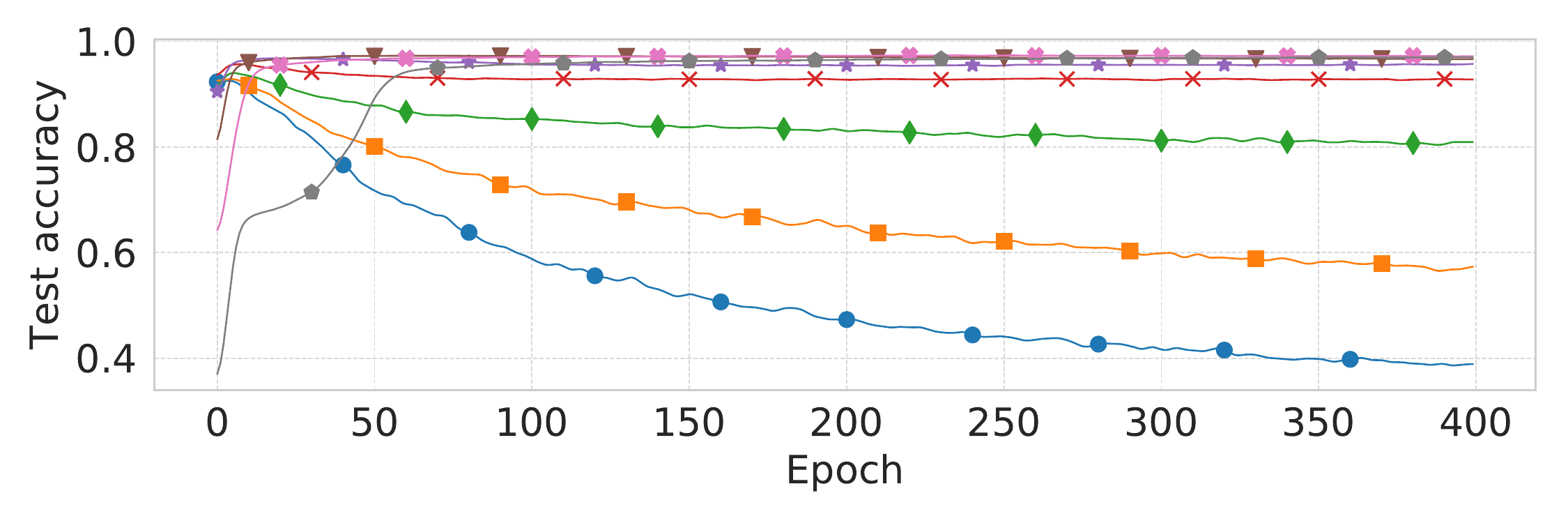}
    \caption{Testing performance LIMIT$_\mathcal{L} - S$}
    \end{subfigure}
    \caption{Training and testing accuracies of ``LIMIT$_\mathcal{G} - S$'' and ``LIMIT$_\mathcal{L} - S$'' instances with varying values of $\beta$ on MNIST with 80\% uniform label noise.
    The curves are smoothed for better presentation.}
    \label{fig:mnist-error-noise80-beta}
\end{figure*}

\paragraph{Detecting incorrect samples.}
In the proposed approach, the auxiliary network $q$ should not be able to distinguish correct and incorrect samples, unless it overfits.
In fact, Fig.~\ref{fig:grad-norm-hist} shows that if we look at the norm of predicted gradients, examples with correct and incorrect labels are indistinguishable in easy cases (MNIST with 80\% uniform noise and CIFAR-10 with 40\% uniform noise) and have large overlap in harder cases (CIFAR-10 with 40\% pair noise and CIFAR-100 with 40\% uniform noise).
Therefore, we hypothesize that the auxiliary network learns to utilize incorrect samples effectively by predicting ``correct'' gradients.
This also hints that the distance between the predicted and cross-entropy gradients might be useful for detecting samples with incorrect or confusing labels.
Fig.~\ref{fig:grad-diff-hist} confirms this intuition, demonstrating that this distance separates correct and incorrect samples perfectly in easy cases (MNIST with 80\% uniform noise and CIFAR-10 with 40\% uniform noise) and separates them well in harder cases (CIFAR-10 with 40\% pair noise and CIFAR-100 with 40\% uniform noise).
If we interpret this distance as a score for classifying correctness of a label, we get 91.1\% ROC AUC score in the hardest case: CIFAR-10 with 40\% pair noise, and more than 99\% score in the easier cases.
Motivated by this results, we use this analysis to detect samples with incorrect or confusing labels in the original MNIST, CIFAR-10, and Clothing1M datasets.
We present a few incorrect/confusing labels for each class in Figures \ref{fig:mnist-cifar-confusing-more-examples} and \ref{fig:clothgin1m-confusing-more-examples}.

\paragraph{Quantitative results.} Tables \ref{tab:mnist-with-error-bars-1}, \ref{tab:mnist-with-error-bars-2}, \ref{tab:cifar10_error_with_error_bars}, and \ref{tab:cifar10_custom_with_error_bars} present test accuracy comparisons on multiple corrupted versions of MNIST and CIFAR-10.
The presented error bars are standard deviations.
In case of MNIST, we compute them over 5 training/validation splits.
In the case of CIFAR-10, due to high computational cost, we have only one run for each model and dataset pair.
The standard deviations are computed by resampling the corresponding test sets 1000 times with replacement.

\begin{figure}[t]
    \captionsetup[subfigure]{justification=centering}
    \centering
    \begin{subfigure}{0.242\textwidth}
    \centering
    \includegraphics[width=\textwidth]{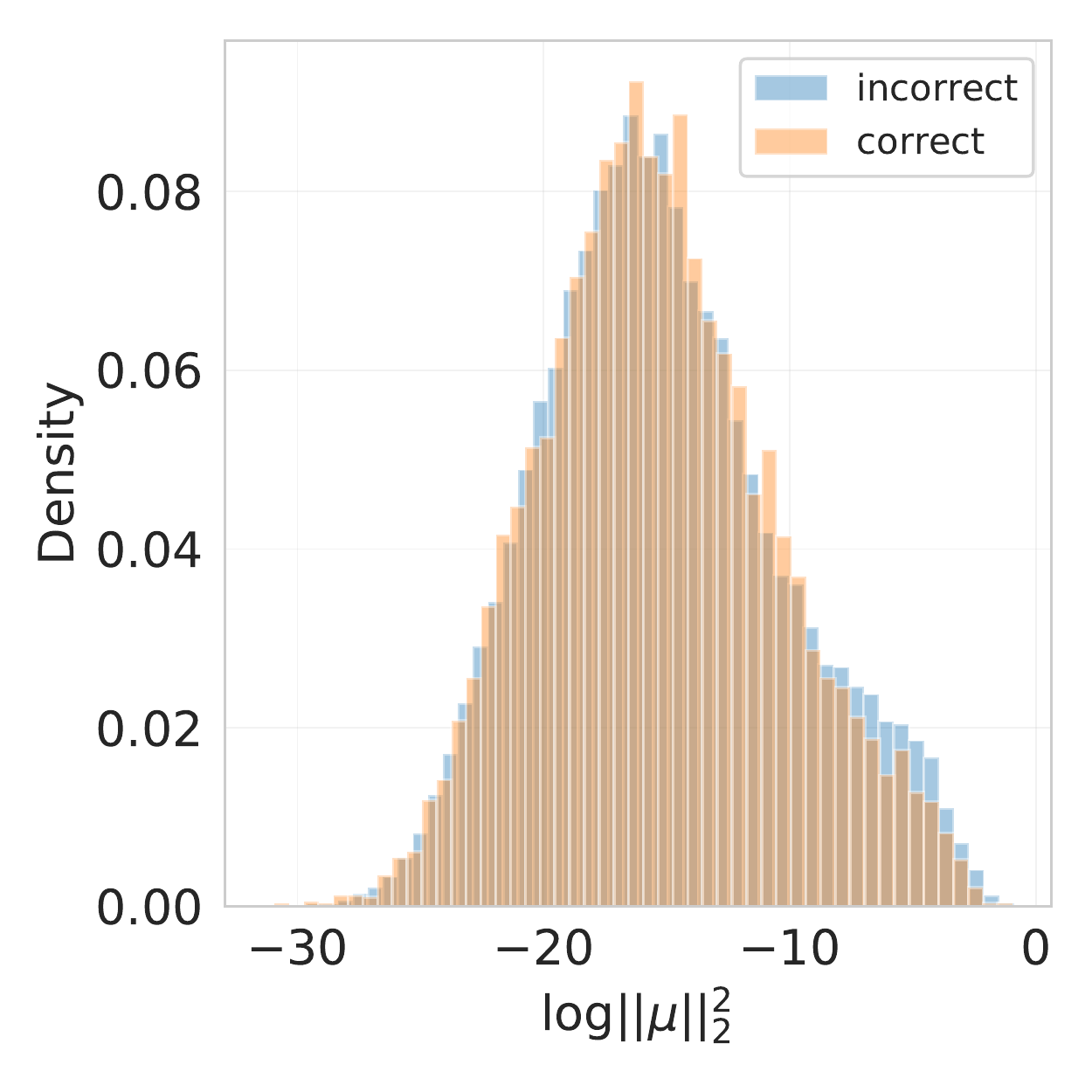}
    \caption{MNIST\\80\% uniform noise}
    \end{subfigure}%
    ~
    \begin{subfigure}{0.242\textwidth}
    \includegraphics[width=\textwidth]{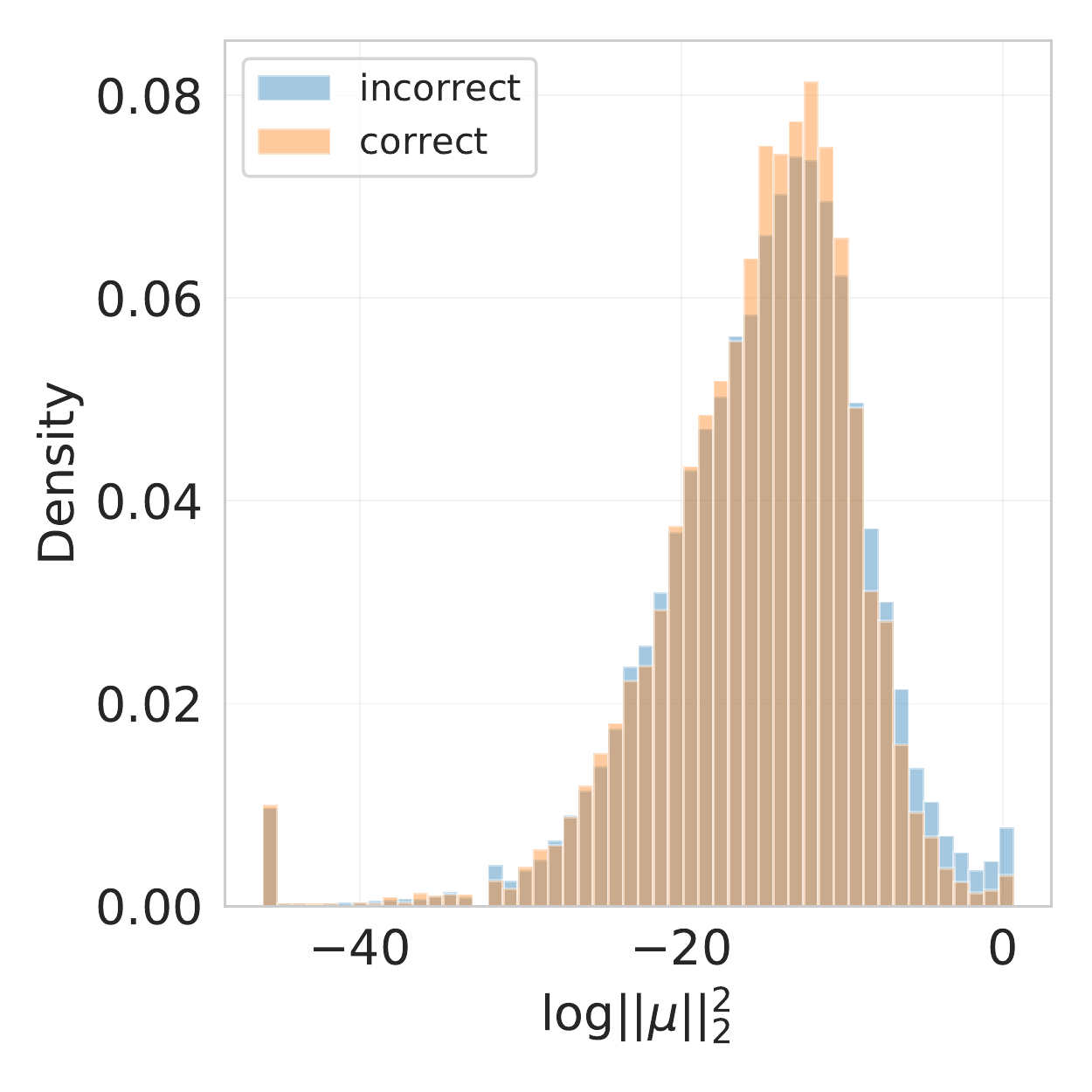}
    \caption{CIFAR-10\\40\% uniform noise}
    \end{subfigure}%
    ~
    \begin{subfigure}{0.242\textwidth}
    \includegraphics[width=\textwidth]{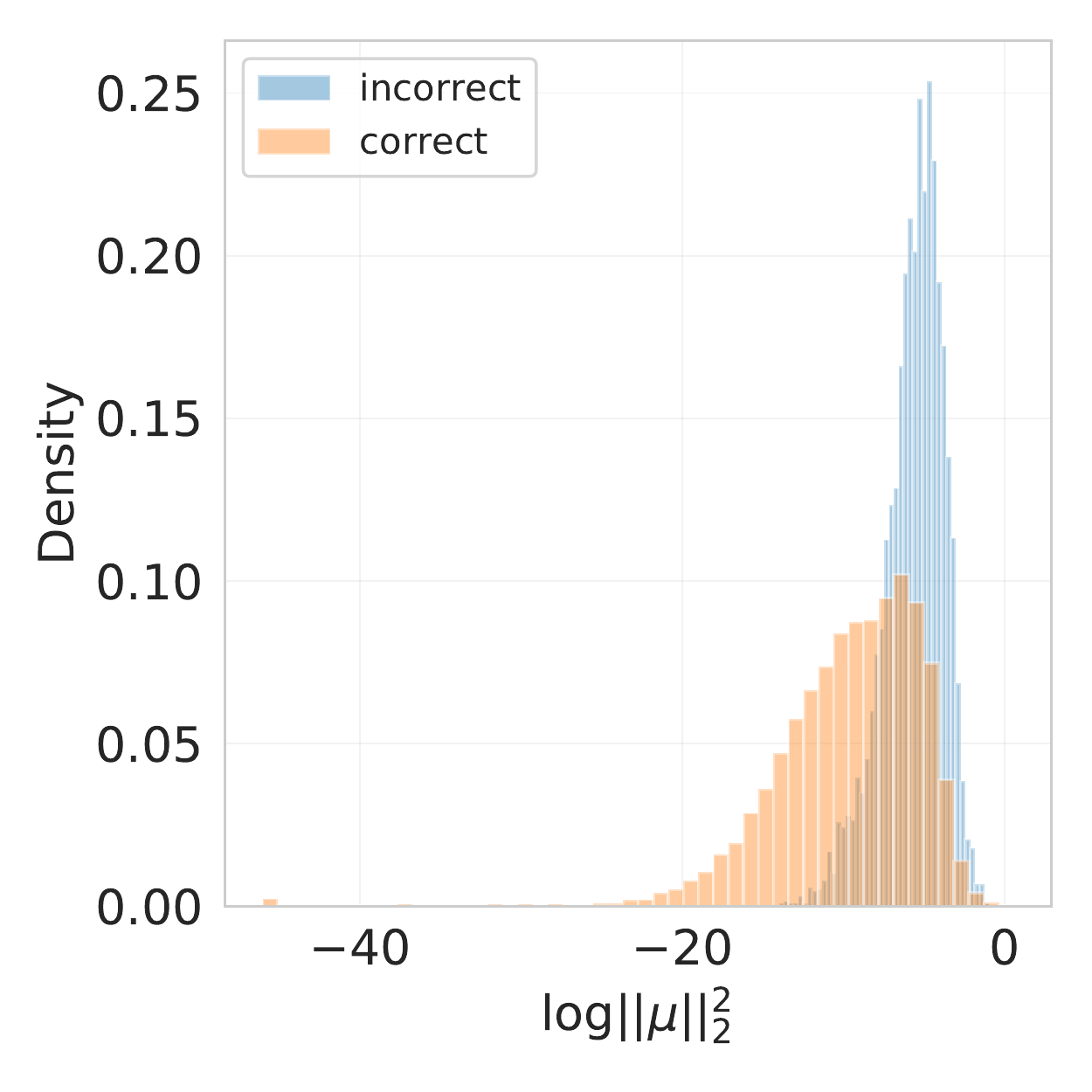}
    \caption{CIFAR-10\\40\% pair noise}
    \end{subfigure}%
    ~
    \begin{subfigure}{0.242\textwidth}
    \includegraphics[width=\textwidth]{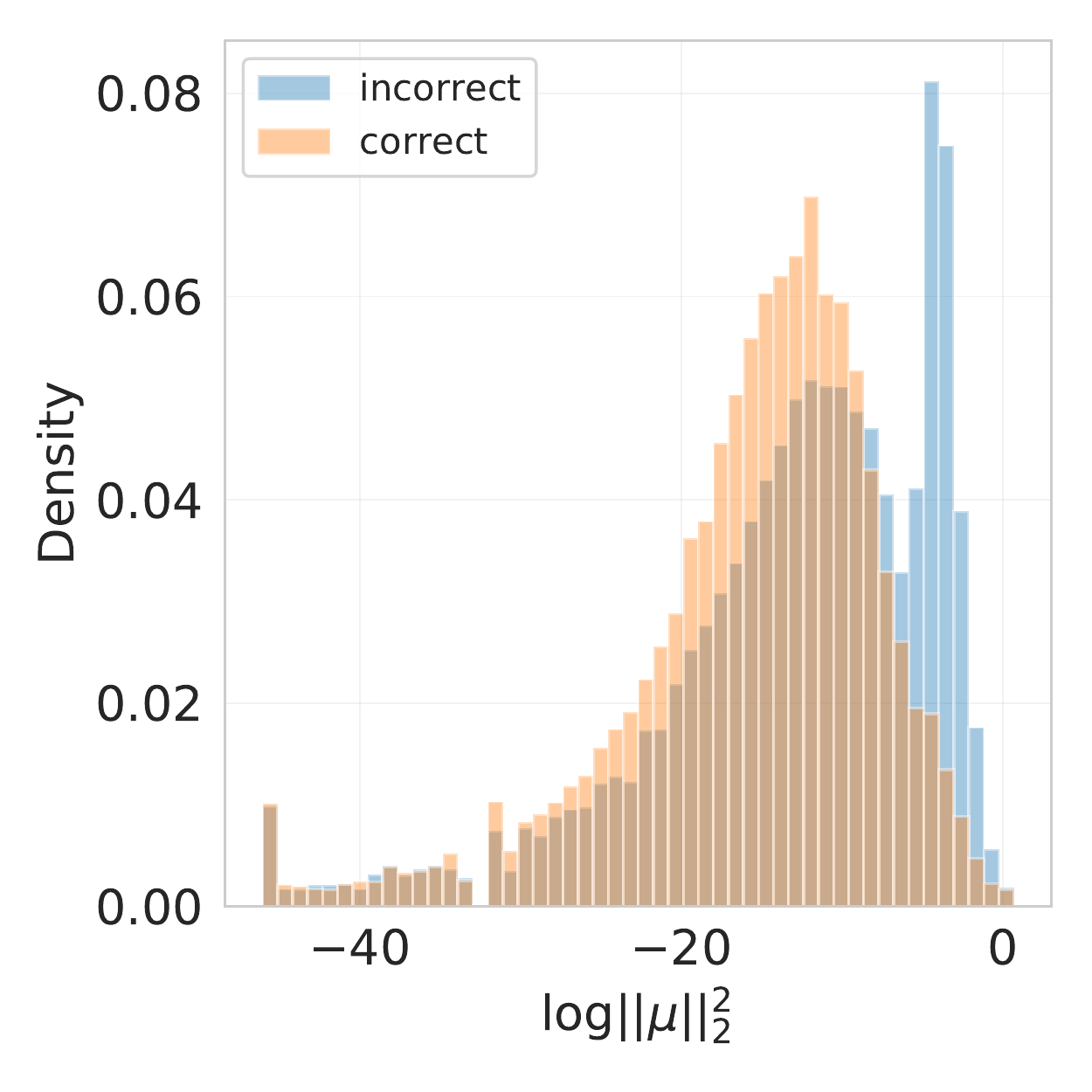}
    \caption{CIFAR-100\\40\% uniform noise}
    \end{subfigure}
    \caption{Histograms of the norm of predicted gradients for examples with correct and incorrect labels. The gradient predictions are done using the best instances of LIMIT.}
    \label{fig:grad-norm-hist}
\end{figure}

\begin{figure}[t]
    \captionsetup[subfigure]{justification=centering}
    \centering
    \begin{subfigure}{0.242\textwidth}
    \centering
    \includegraphics[width=\textwidth]{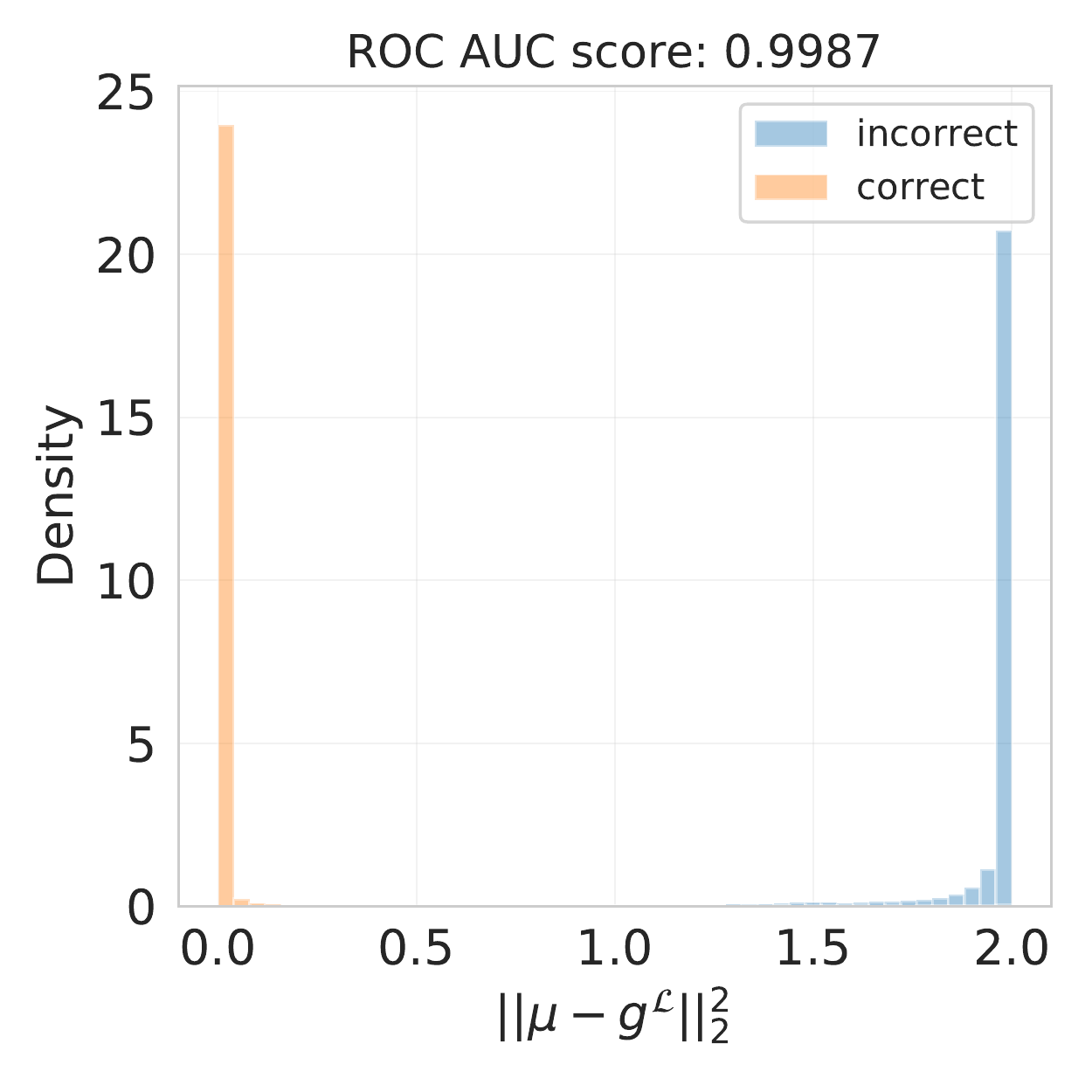}
    \caption{MNIST\\80\% uniform noise}
    \end{subfigure}%
    ~
    \begin{subfigure}{0.242\textwidth}
    \includegraphics[width=\textwidth]{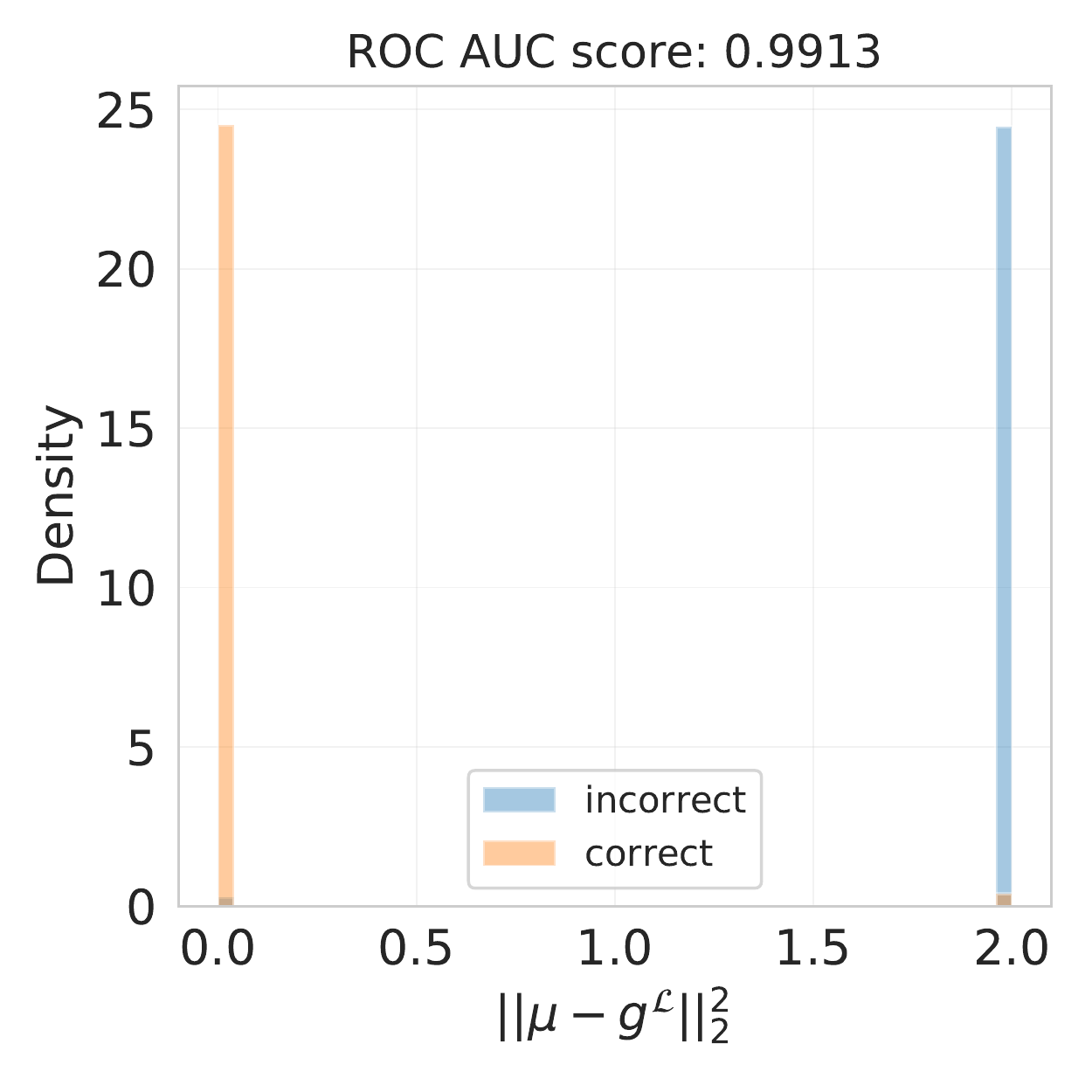}
    \caption{CIFAR-10\\40\% uniform noise}
    \end{subfigure}%
    ~
    \begin{subfigure}{0.242\textwidth}
    \includegraphics[width=\textwidth]{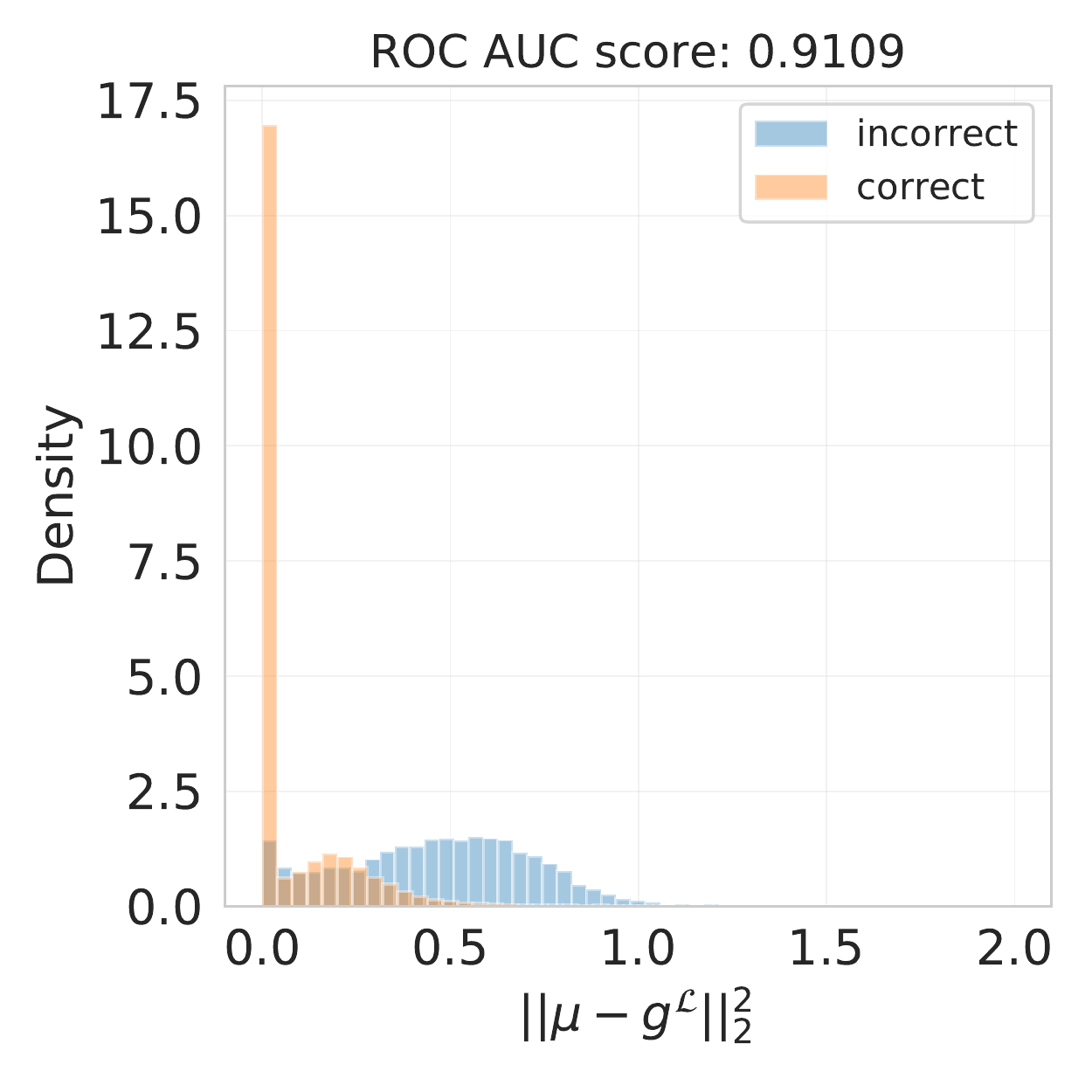}
    \caption{CIFAR-10\\40\% pair noise}
    \end{subfigure}%
    ~
    \begin{subfigure}{0.242\textwidth}
    \includegraphics[width=\textwidth]{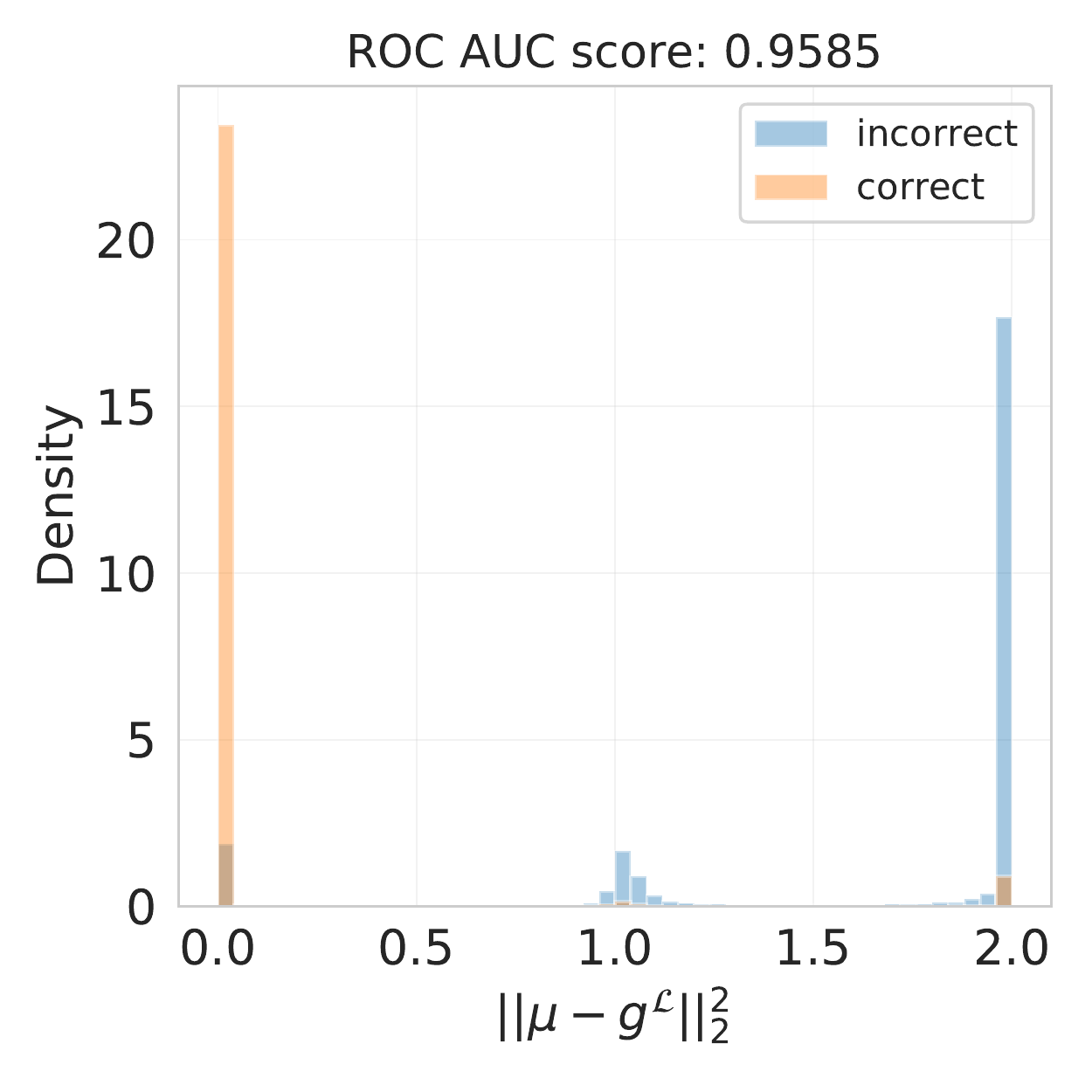}
    \caption{CIFAR-100\\40\% uniform noise}
    \end{subfigure}
    \caption{Histograms of the distance between predicted and actual gradient for examples with correct and incorrect labels. The gradient predictions are done using the best instances of LIMIT.}
    \label{fig:grad-diff-hist}
\end{figure}

\begin{table*}[t]
\small
\begin{center}
\begin{tabular}{lcccccc}
\toprule
\multirow{2}{*}{Method} & \multicolumn{3}{c}{$p=0.0$} & \multicolumn{3}{c}{$p=0.5$}\\
\cmidrule(lr){2-4}
\cmidrule(lr){5-7}
& $n=10^3$ & $n=10^4$ & All & $n=10^3$ & $n=10^4$ & All\\
\midrule
CE                             & 94.3 $\pm$  0.5 & 98.4 $\pm$  0.2 & \textbf{99.2 $\pm$  0.0} & 71.8 $\pm$  4.3 & 93.1 $\pm$  0.6 & 97.2 $\pm$  0.2\\
CE + GN         & 89.5 $\pm$  0.8 & 95.4 $\pm$  0.5 & 97.1 $\pm$  0.5 & 70.5 $\pm$  3.5 & 92.3 $\pm$  0.7 & 97.4 $\pm$  0.5\\
CE + LN          & 90.0 $\pm$  0.5 & 95.3 $\pm$  0.6 & 96.7 $\pm$  0.7 & 66.8 $\pm$  1.3 & 92.0 $\pm$  1.5 & 97.6 $\pm$  0.1\\
MAE                            & 94.6 $\pm$  0.5 & 98.3 $\pm$  0.2 & \textbf{99.1 $\pm$  0.1} & 75.6 $\pm$  5.0 & 95.7 $\pm$  0.5 & 98.1 $\pm$  0.1\\
FW                             & 93.6 $\pm$  0.6 & 98.4 $\pm$  0.1 & \textbf{99.2 $\pm$  0.1} & 64.3 $\pm$  9.1 & 91.6 $\pm$  2.0 & 97.3 $\pm$  0.3\\
DMI                            & 94.5 $\pm$  0.5 & 98.5 $\pm$  0.1 & \textbf{99.2 $\pm$  0.0} & 79.8 $\pm$  2.9 & 95.7 $\pm$  0.3 & 98.3 $\pm$  0.1\\
Soft reg. (\ref{eq:penalize-final})                       & \textbf{95.7 $\pm$  0.2} & 98.4 $\pm$  0.1 & \textbf{99.2 $\pm$  0.0} & 76.4 $\pm$  2.4 & 95.7 $\pm$  0.0 & 98.2 $\pm$  0.1\\
LIMIT$_\mathcal{G}$ + S        & \textbf{95.6 $\pm$  0.3} & \textbf{98.6 $\pm$  0.1} & \textbf{99.3 $\pm$  0.0} & 82.8 $\pm$  4.6 & 97.0 $\pm$  0.1 & 98.7 $\pm$  0.1\\
LIMIT$_\mathcal{L}$ + S         & 94.8 $\pm$  0.3 & \textbf{98.6 $\pm$  0.2} & \textbf{99.3 $\pm$  0.0} & \textbf{88.7 $\pm$  3.8} & \textbf{97.6 $\pm$  0.1} & \textbf{98.9 $\pm$  0.0}\\
LIMIT$_\mathcal{G}$ - S               & \textbf{95.7 $\pm$  0.2} & \textbf{98.7 $\pm$  0.1} & \textbf{99.3 $\pm$  0.1} & 83.3 $\pm$  2.3 & 97.1 $\pm$  0.2 & 98.6 $\pm$  0.1\\
LIMIT$_\mathcal{L}$ - S                & 95.0 $\pm$  0.2 & \textbf{98.7 $\pm$  0.1} & \textbf{99.3 $\pm$  0.1} & \textbf{88.2 $\pm$  2.9} & \textbf{97.7 $\pm$  0.1} & \textbf{99.0 $\pm$  0.1}\\
\bottomrule
\end{tabular}
\end{center}
\caption{Test accuracy comparison on multiple versions of MNIST corrupted with uniform label noise.}
\label{tab:mnist-with-error-bars-1}
\end{table*}

\begin{table*}[!t]
\small
\begin{center}
\begin{tabular}{lcccccc}
\toprule
\multirow{2}{*}{Method} & \multicolumn{3}{c}{$p=0.8$} & \multicolumn{3}{c}{$p=0.89$}\\
\cmidrule(lr){2-4}
\cmidrule(lr){5-7}
& $n=10^3$ & $n=10^4$ & All & $n=10^3$ & $n=10^4$ & All\\
\midrule
CE                             & 27.0 $\pm$  3.8 & 69.9 $\pm$  2.6 & 87.2 $\pm$  1.0 & 10.3 $\pm$  1.6 & 13.4 $\pm$  3.3 & 13.2 $\pm$  1.8\\
CE + GN         & 25.9 $\pm$  4.6 & 51.9 $\pm$ 10.5 & 85.3 $\pm$  8.3 & 10.4 $\pm$  4.5 & 10.2 $\pm$  3.3 & 11.1 $\pm$  0.4\\
CE + LN          & 30.2 $\pm$  4.8 & 53.1 $\pm$  6.4 & 74.5 $\pm$ 19.1 & 11.9 $\pm$  3.9 &  8.8 $\pm$  5.4 & 14.1 $\pm$  4.3\\
MAE                            & 25.1 $\pm$  3.3 & 74.6 $\pm$  2.7 & 93.2 $\pm$  1.1 & 10.9 $\pm$  1.4 & 12.1 $\pm$  3.9 & 17.6 $\pm$  8.1\\
FW                             & 19.0 $\pm$  4.1 & 61.2 $\pm$  5.0 & 89.1 $\pm$  2.1 &  8.7 $\pm$  2.8 & 11.4 $\pm$  1.4 & 12.3 $\pm$  1.8\\
DMI                            & 30.3 $\pm$  5.1 & 79.0 $\pm$  1.5 & 88.8 $\pm$  0.9 & 10.5 $\pm$  1.2 & 14.1 $\pm$  5.1 & 12.5 $\pm$  1.5\\
Soft reg. (\ref{eq:penalize-final})                       & 28.8 $\pm$  2.2 & 67.0 $\pm$  1.9 & 89.3 $\pm$  0.6 & 10.3 $\pm$  1.6 & 10.5 $\pm$  0.8 & 12.7 $\pm$  2.6\\
LIMIT$_\mathcal{G}$ + S        & \textbf{35.9 $\pm$  6.3} & 80.6 $\pm$  2.8 & 93.4 $\pm$  0.5 & 10.0 $\pm$  1.0 & 14.3 $\pm$  5.4 & 13.1 $\pm$  4.3\\
LIMIT$_\mathcal{L}$ + S         & \textbf{35.6 $\pm$  3.2} & \textbf{93.3 $\pm$  0.3} & \textbf{97.6 $\pm$  0.3} & 10.1 $\pm$  0.7 & 12.5 $\pm$  2.1 & \textbf{28.3 $\pm$  8.1}\\
LIMIT$_\mathcal{G}$ - S               & \textbf{37.1 $\pm$  5.4} & 82.0 $\pm$  1.5 & 94.7 $\pm$  0.6 &  9.9 $\pm$  1.0 & 12.6 $\pm$  0.3 & 16.0 $\pm$  5.9\\
LIMIT$_\mathcal{L}$ - S                & \textbf{35.9 $\pm$  4.3} & \textbf{93.9 $\pm$  0.8} & \textbf{97.7 $\pm$  0.2} & 11.1 $\pm$  0.7 & 11.8 $\pm$  1.0 & \textbf{28.6 $\pm$  4.0}\\
\bottomrule
\end{tabular}
\end{center}
\caption{Test accuracy comparison on multiple versions of MNIST corrupted with uniform label noise.}
\label{tab:mnist-with-error-bars-2}
\end{table*}
\begin{table}[t]
    \small
    \centering
    \begin{tabular}{lccccc}
    \toprule
    Method & $p=0.0$ & $p=0.2$ & $p=0.4$  & $p=0.6$  & $p=0.8$\\
    \midrule
    CE                             & 92.7 $\pm$  0.3 & 85.2 $\pm$  0.4 & 81.0 $\pm$  0.4 & 69.0 $\pm$  0.5 & 38.8 $\pm$  0.5\\
    MAE                            & 84.4 $\pm$  0.4 & 85.4 $\pm$  0.4 & 64.6 $\pm$  0.5 & 15.4 $\pm$  0.4 & 12.0 $\pm$  0.3\\
    FW                             & 92.9 $\pm$  0.3 & 86.2 $\pm$  0.3 & 81.4 $\pm$  0.4 & 69.7 $\pm$  0.5 & 34.4 $\pm$  0.5\\
    DMI                            & 93.0 $\pm$  0.3 & 88.3 $\pm$  0.3 & 85.0 $\pm$  0.3 & 72.5 $\pm$  0.4 & 38.9 $\pm$  0.5\\
    LIMIT$_\mathcal{G}$               & \textbf{93.5 $\pm$  0.2} & 90.7 $\pm$  0.3 & 86.6 $\pm$  0.3 & 73.7 $\pm$  0.4 & 38.7 $\pm$  0.5\\
    LIMIT$_\mathcal{L}$                & 93.1 $\pm$  0.3 & 91.5 $\pm$  0.3 & 88.2 $\pm$  0.3 & 75.7 $\pm$  0.4 & 35.8 $\pm$  0.5\\
    LIMIT$_\mathcal{G}$ + init.        & \textbf{93.3 $\pm$  0.3} & \textbf{92.4 $\pm$  0.3} & \textbf{90.3 $\pm$  0.3} & 81.9 $\pm$  0.4 & \textbf{44.1 $\pm$  0.5}\\
    LIMIT$_\mathcal{L}$ + init.         & \textbf{93.3 $\pm$  0.2} & \textbf{92.2 $\pm$  0.3} & \textbf{90.2 $\pm$  0.3} & \textbf{82.9 $\pm$  0.4} & \textbf{44.3 $\pm$  0.5}\\
    \bottomrule
    \end{tabular}
    \caption{Test accuracy comparison on CIFAR-10, corrupted with uniform label noise. The error bars are computed by bootstrapping the test set 1000 times.}
    \label{tab:cifar10_error_with_error_bars}
\end{table}

\begin{table}[t]
    \small
    \centering
    \begin{tabular}{lccccc}
    \toprule
    Method & $p=0.0$ & $p=0.1$ & $p=0.2$  & $p=0.3$  & $p=0.4$\\
    \midrule
    CE                             & 92.7 $\pm$  0.3 & 90.0 $\pm$  0.3 & 88.1 $\pm$  0.3 & 87.2 $\pm$  0.3 & 81.8 $\pm$  0.4\\
    MAE                            & 84.4 $\pm$  0.4 & 88.6 $\pm$  0.3 & 83.2 $\pm$  0.4 & 72.1 $\pm$  0.4 & 61.1 $\pm$  0.5\\
    FW                             & 92.9 $\pm$  0.3 & 90.1 $\pm$  0.3 & 88.0 $\pm$  0.3 & 86.8 $\pm$  0.3 & 84.6 $\pm$  0.3\\
    DMI                            & 93.0 $\pm$  0.3 & 91.4 $\pm$  0.3 & 90.6 $\pm$  0.3 & 90.4 $\pm$  0.3 & \textbf{89.6 $\pm$  0.3}\\
    LIMIT$_\mathcal{G}$               & \textbf{93.5 $\pm$  0.2} & 92.8 $\pm$  0.3 & 91.3 $\pm$  0.3 & 89.2 $\pm$  0.3 & 86.0 $\pm$  0.3\\
    LIMIT$_\mathcal{L}$                & 93.1 $\pm$  0.3 & 91.9 $\pm$  0.3 & 91.1 $\pm$  0.3 & 88.8 $\pm$  0.3 & 84.2 $\pm$  0.4\\
    LIMIT$_\mathcal{G}$ + init.        & \textbf{93.3 $\pm$  0.3} & \textbf{93.3 $\pm$  0.3} & \textbf{92.9 $\pm$  0.3} & \textbf{90.8 $\pm$  0.3} & 88.3 $\pm$  0.3\\
    LIMIT$_\mathcal{L}$ + init.         & \textbf{93.3 $\pm$  0.2} & \textbf{93.0 $\pm$  0.2} & 92.3 $\pm$  0.3 & \textbf{91.1 $\pm$  0.3} & \textbf{90.0 $\pm$  0.3}\\
    \bottomrule
    \end{tabular}
    \caption{Test accuracy comparison on CIFAR-10, corrupted with pair noise, described in Sec.~\ref{subsec:cifar-pair}. The error bars are computed by bootstrapping the test set 1000 times.}
    \label{tab:cifar10_custom_with_error_bars}
\end{table}

\begin{figure}
    \centering
    \begin{subfigure}{0.49\textwidth}
    \includegraphics[width=\textwidth]{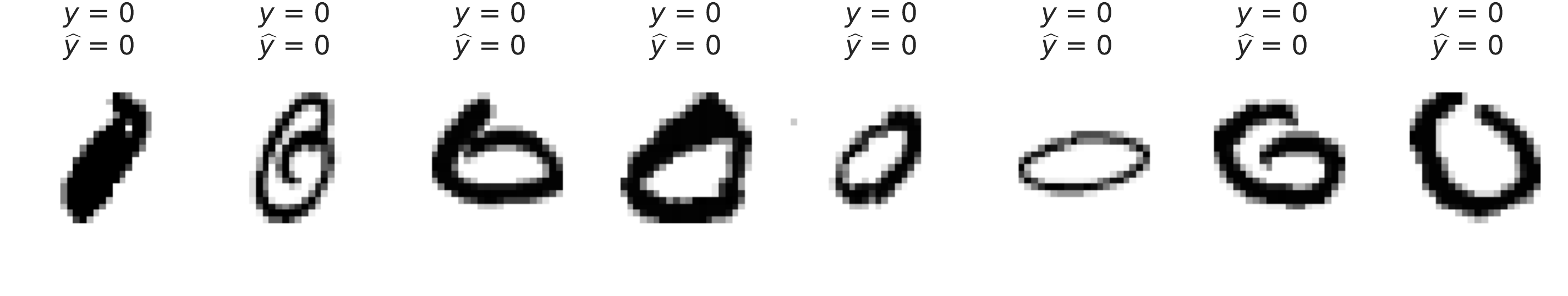}
    \includegraphics[width=\textwidth]{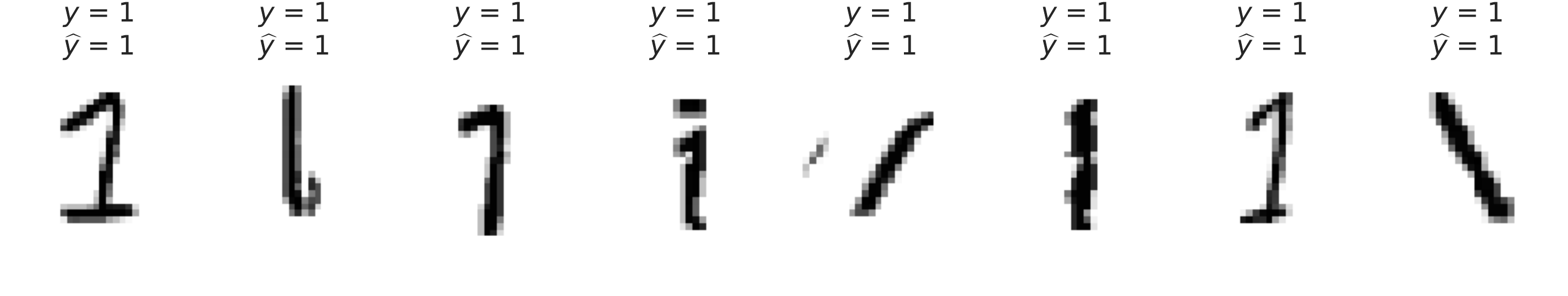}
    \includegraphics[width=\textwidth]{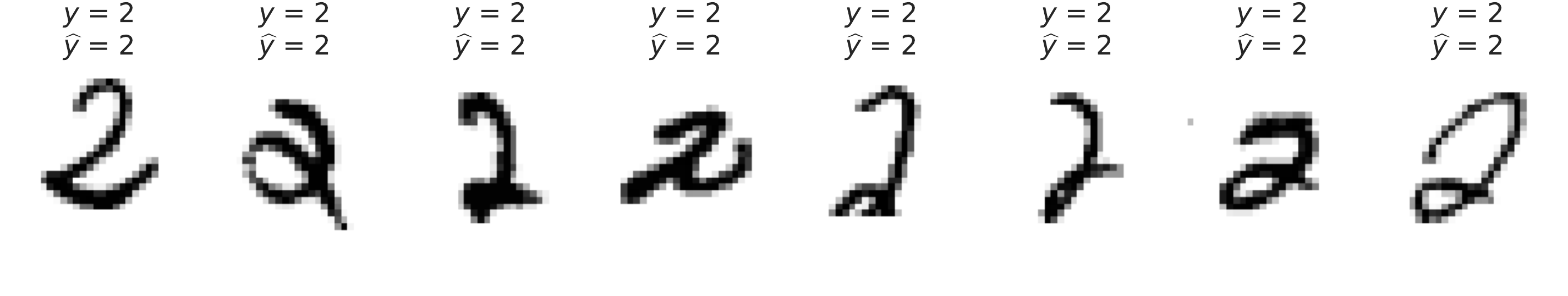}
    \includegraphics[width=\textwidth]{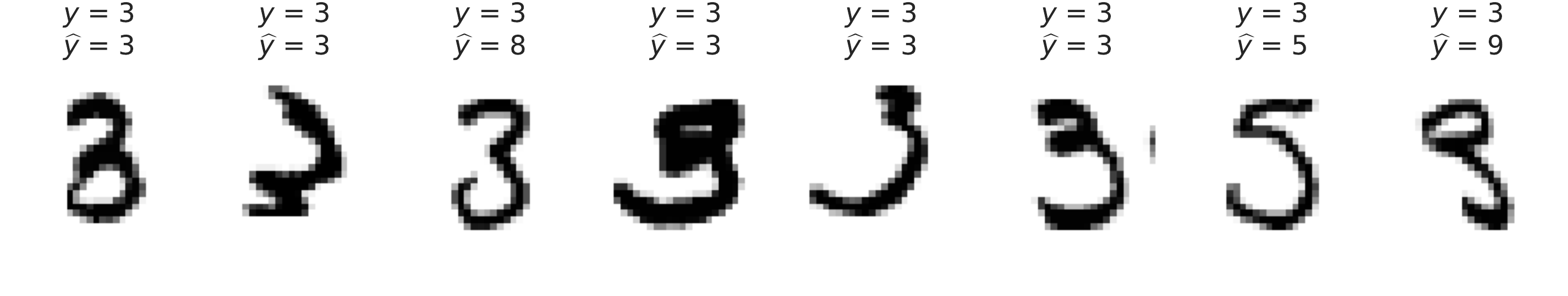}
    \includegraphics[width=\textwidth]{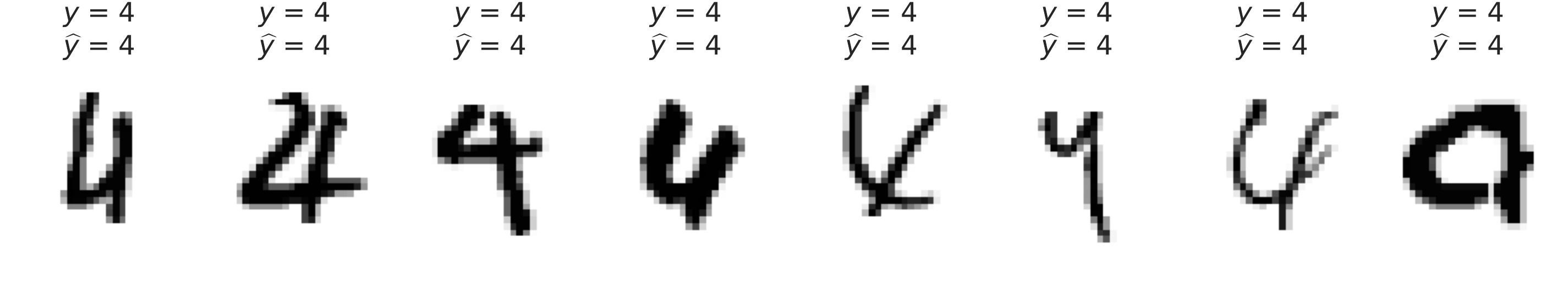}
    \includegraphics[width=\textwidth]{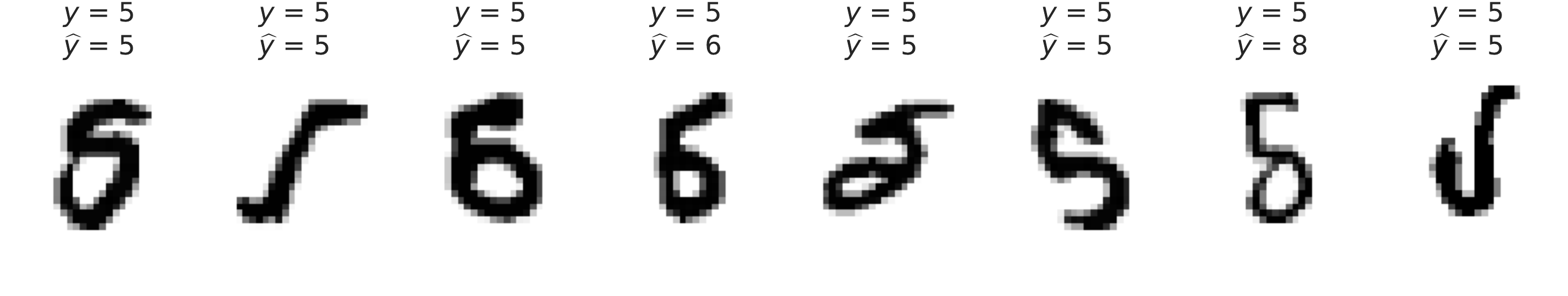}
    \includegraphics[width=\textwidth]{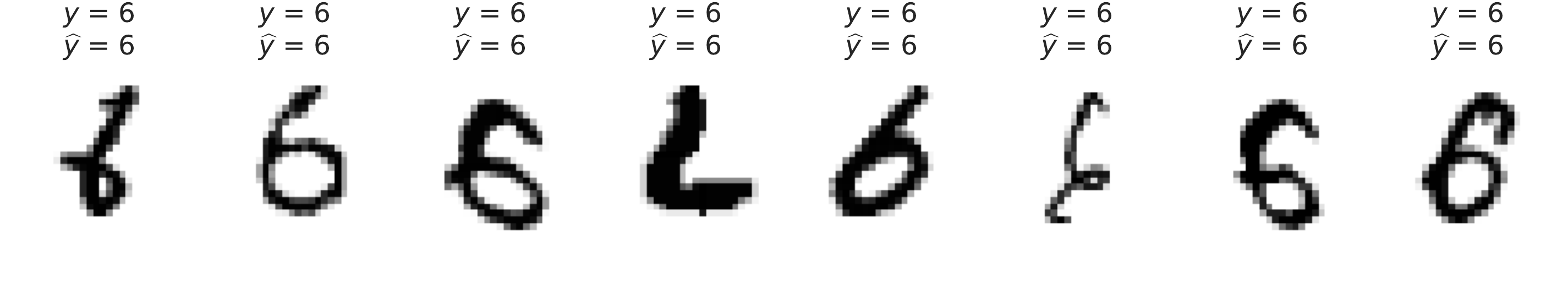}
    \includegraphics[width=\textwidth]{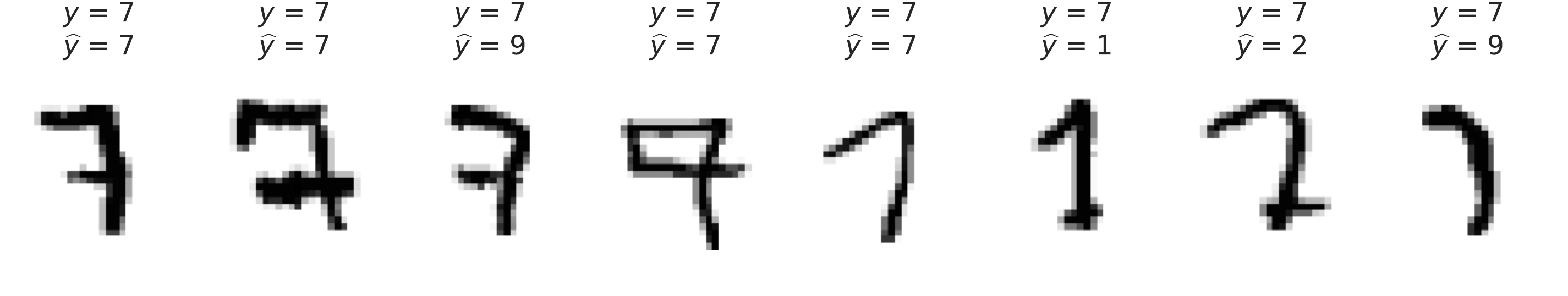}
    \includegraphics[width=\textwidth]{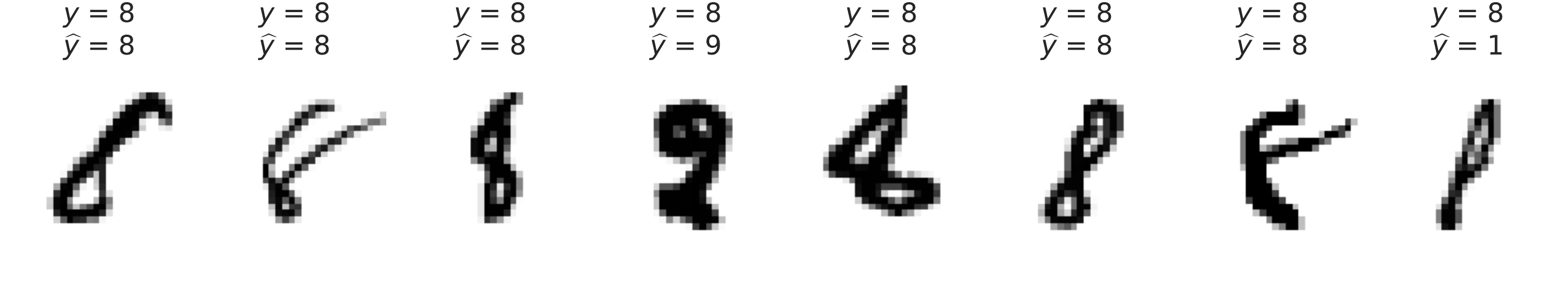}
    \includegraphics[width=\textwidth]{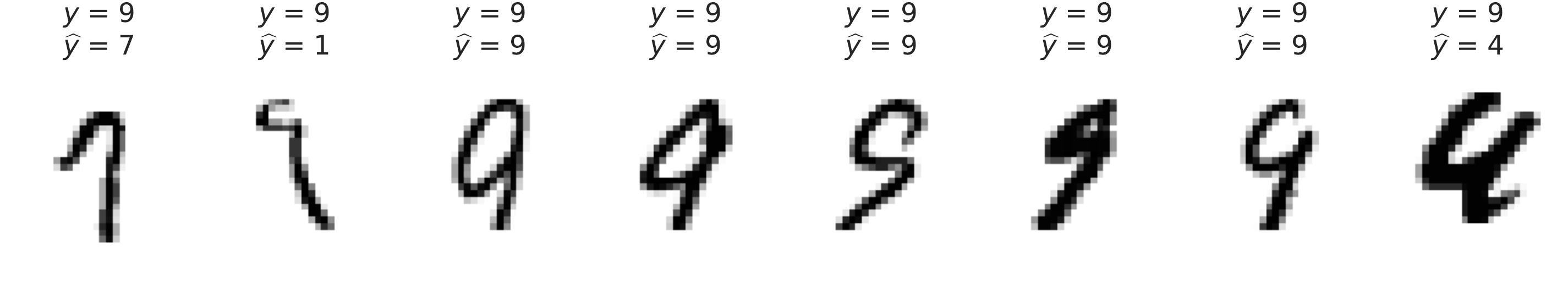}
    \caption{MNIST}
    \end{subfigure}%
    ~
    \begin{subfigure}{0.49\textwidth}
    \includegraphics[width=\textwidth]{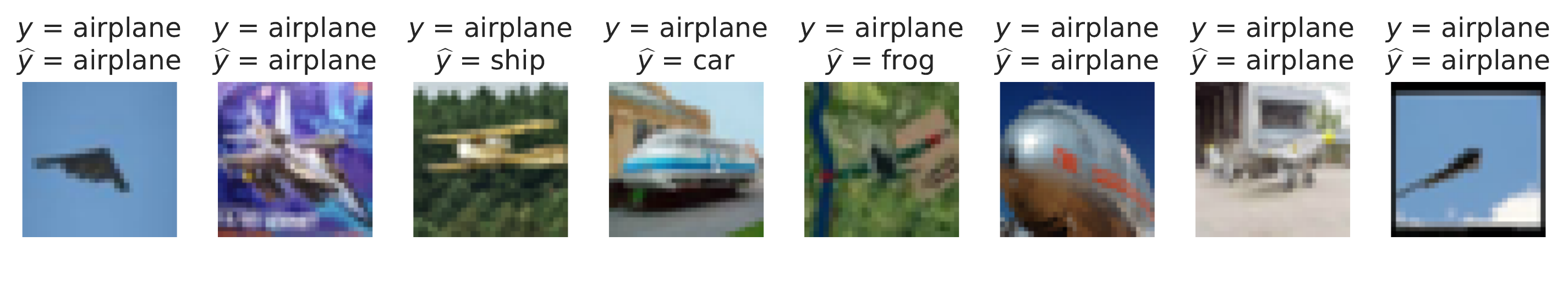}
    \includegraphics[width=\textwidth]{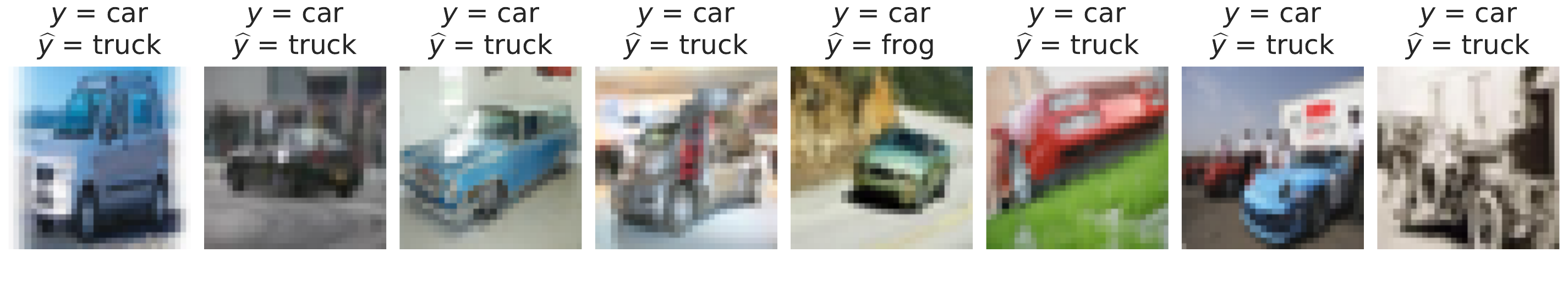}
    \includegraphics[width=\textwidth]{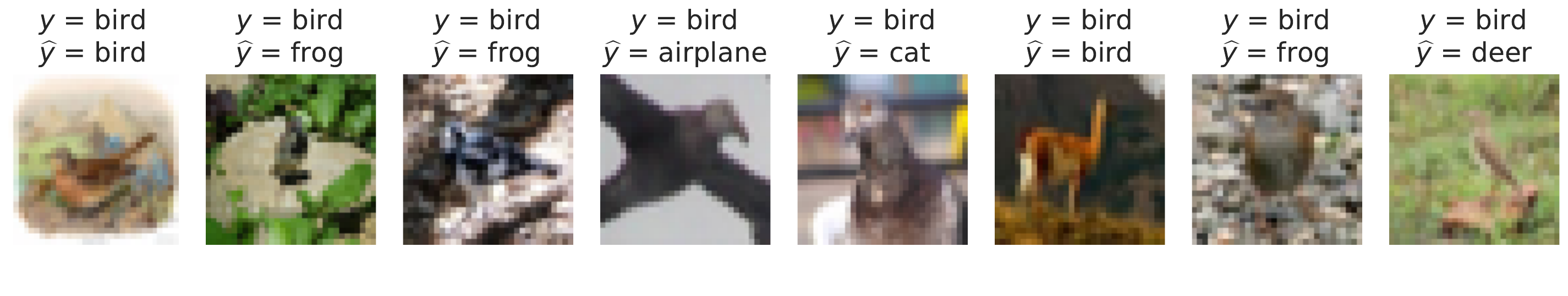}
    \includegraphics[width=\textwidth]{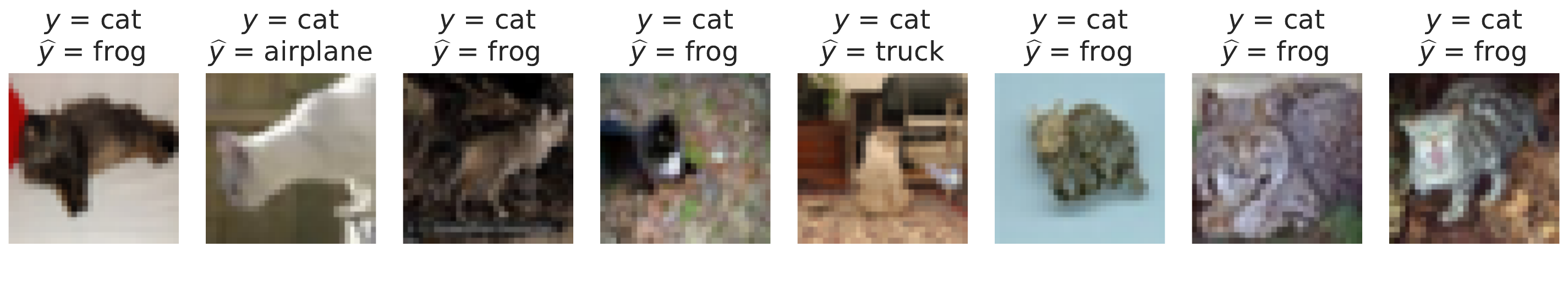}
    \includegraphics[width=\textwidth]{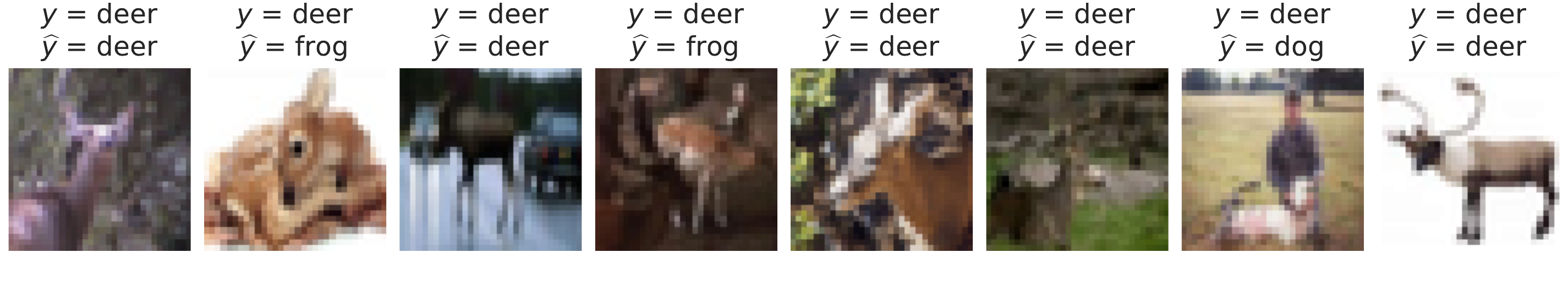}
    \includegraphics[width=\textwidth]{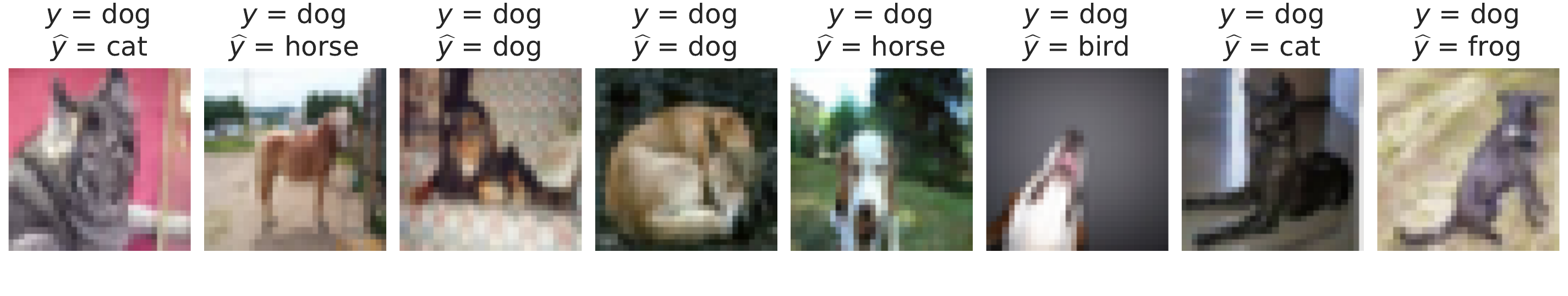}
    \includegraphics[width=\textwidth]{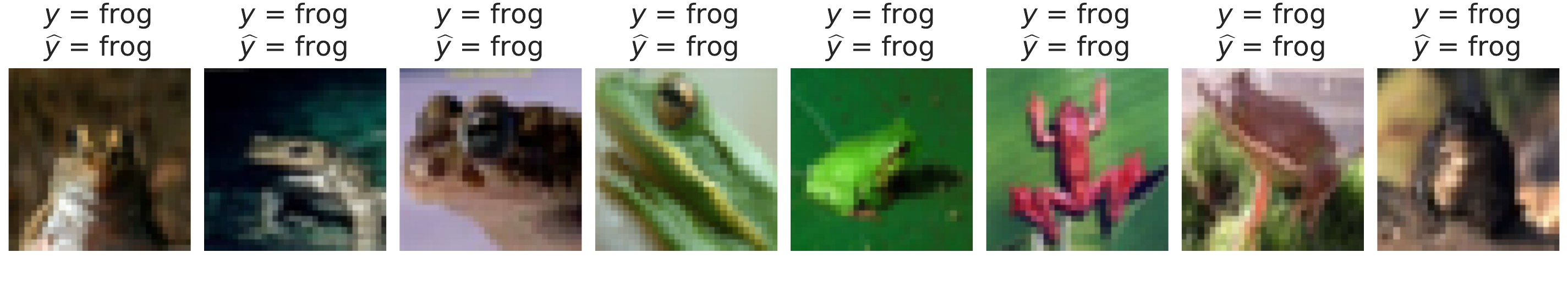}
    \includegraphics[width=\textwidth]{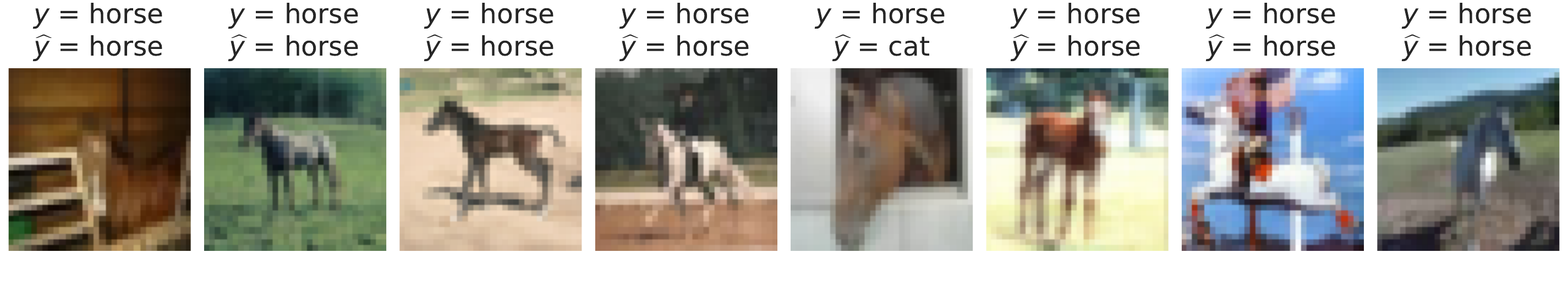}
    \includegraphics[width=\textwidth]{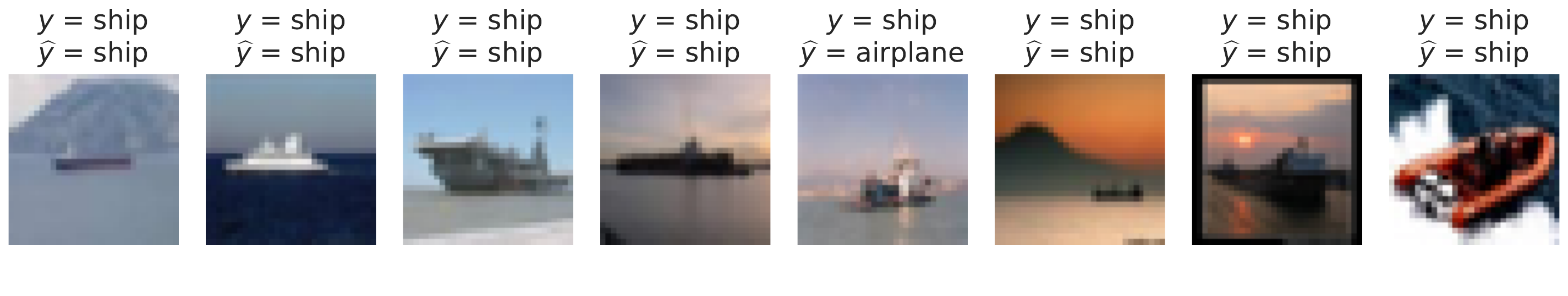}
    \includegraphics[width=\textwidth]{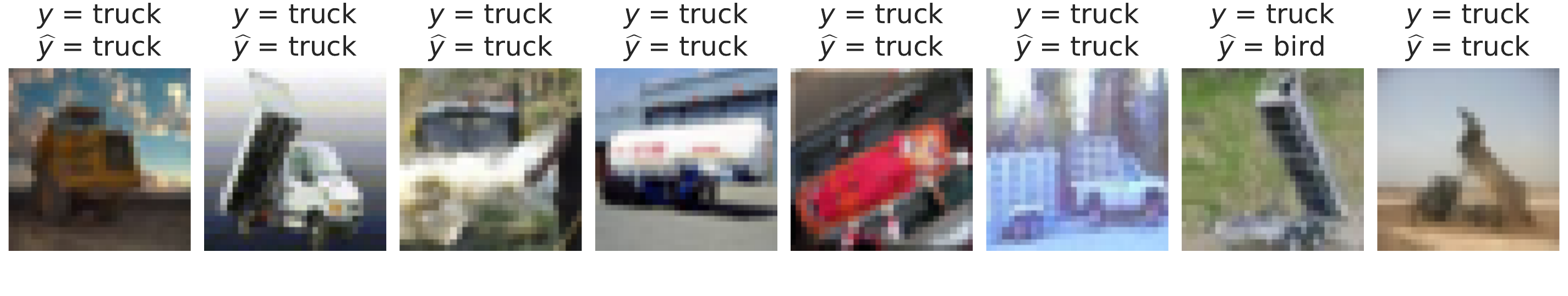}
    \caption{CIFAR-10}
    \end{subfigure}
    \caption{Most confusing 8 labels per class in the MNIST (on the left) and CIFAR-10 (on the right) datasets, according to the distance between predicted and cross-entropy gradients. The gradient predictions are done using the best instances of LIMIT.}
    \label{fig:mnist-cifar-confusing-more-examples}
\end{figure}

\begin{figure}
    \centering
    \begin{subfigure}{0.9\textwidth}
    \includegraphics[width=\textwidth]{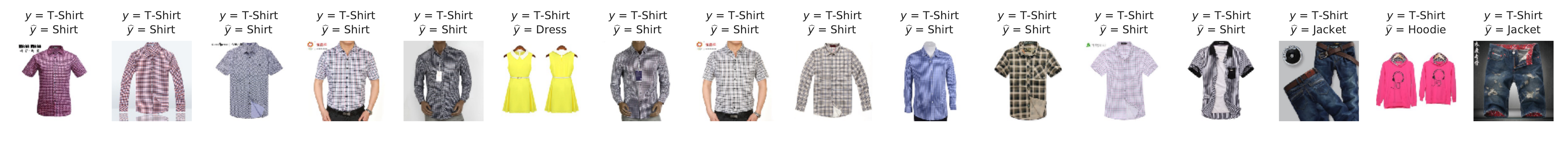}
    \end{subfigure}
    \begin{subfigure}{0.9\textwidth}
    \includegraphics[width=\textwidth]{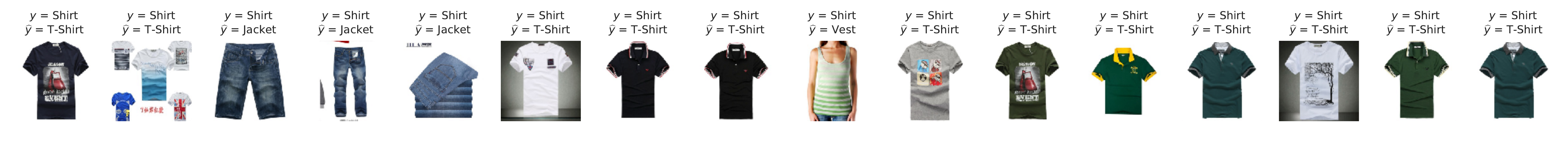}
    \end{subfigure}
    \begin{subfigure}{0.9\textwidth}
    \includegraphics[width=\textwidth]{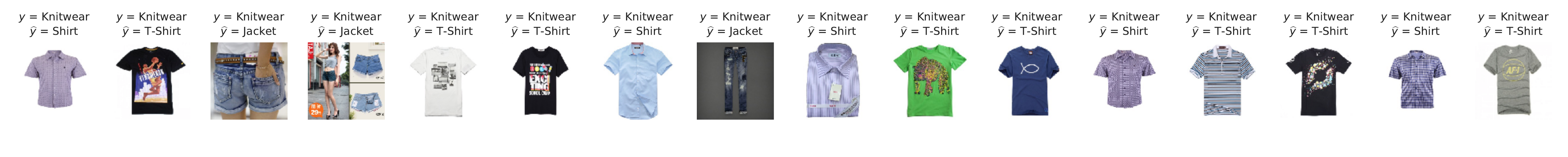}
    \end{subfigure}
    \begin{subfigure}{0.9\textwidth}
    \includegraphics[width=\textwidth]{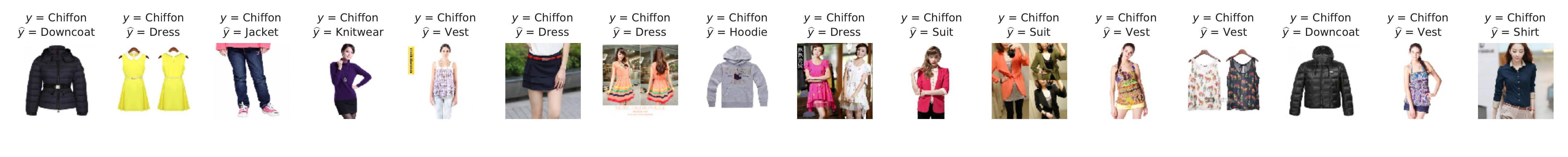}
    \end{subfigure}
    \begin{subfigure}{0.9\textwidth}
    \includegraphics[width=\textwidth]{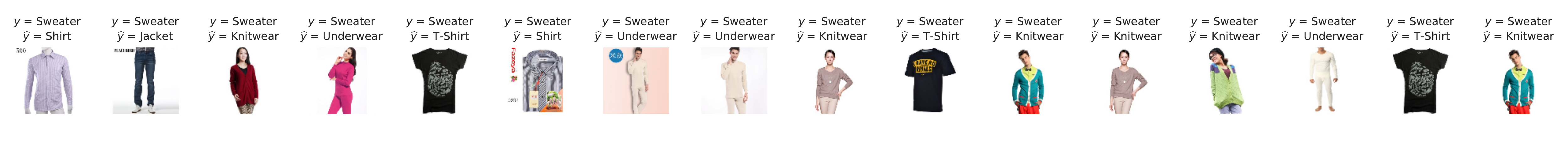}
    \end{subfigure}
    \begin{subfigure}{0.9\textwidth}
    \includegraphics[width=\textwidth]{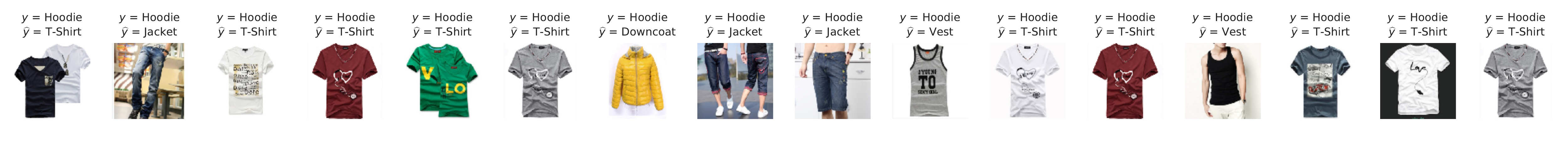}
    \end{subfigure}
    \begin{subfigure}{0.9\textwidth}
    \includegraphics[width=\textwidth]{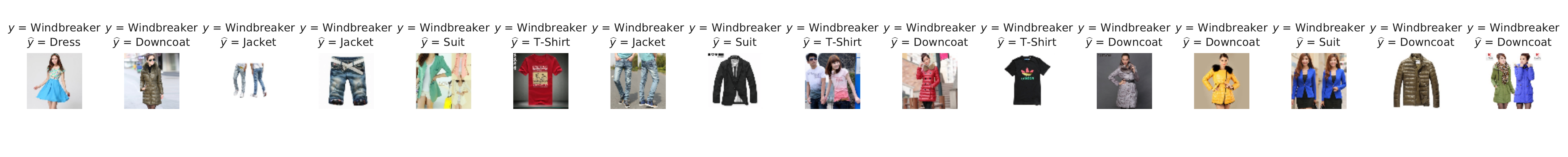}
    \end{subfigure}
    \begin{subfigure}{0.9\textwidth}
    \includegraphics[width=\textwidth]{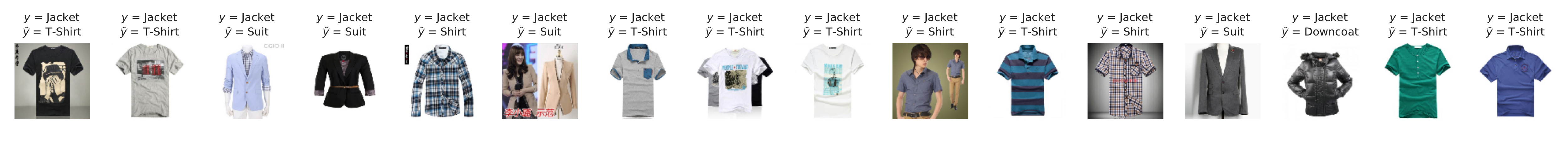}
    \end{subfigure}
    \begin{subfigure}{0.9\textwidth}
    \includegraphics[width=\textwidth]{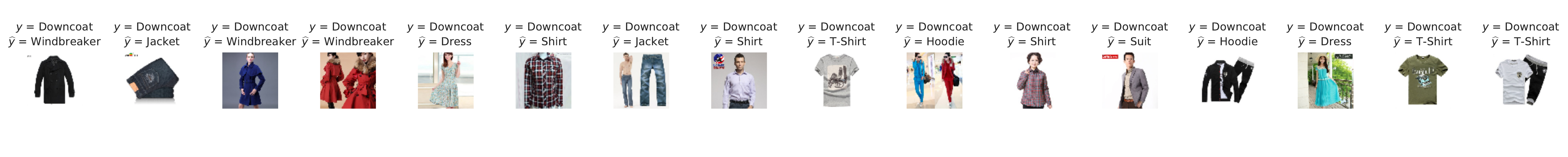}
    \end{subfigure}
    \begin{subfigure}{0.9\textwidth}
    \includegraphics[width=\textwidth]{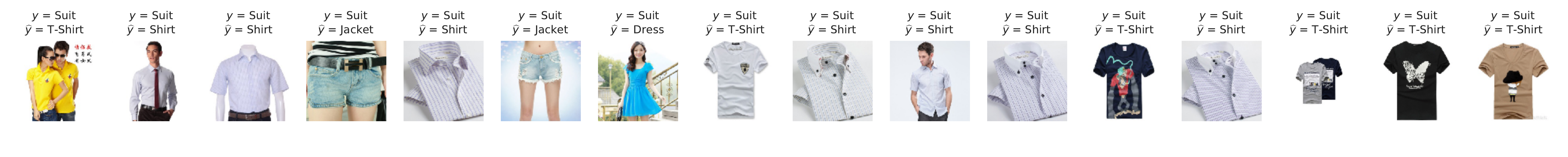}
    \end{subfigure}
    \begin{subfigure}{0.9\textwidth}
    \includegraphics[width=\textwidth]{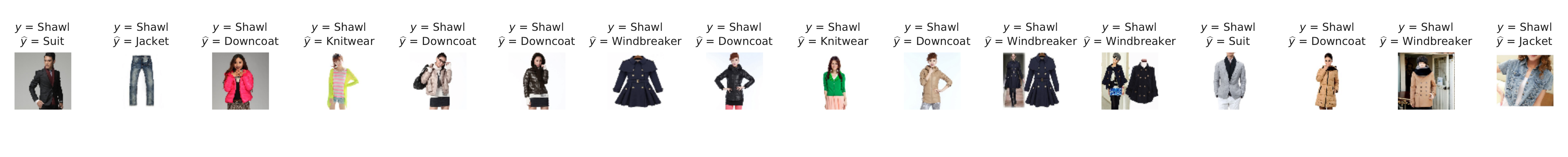}
    \end{subfigure}
    \begin{subfigure}{0.9\textwidth}
    \includegraphics[width=\textwidth]{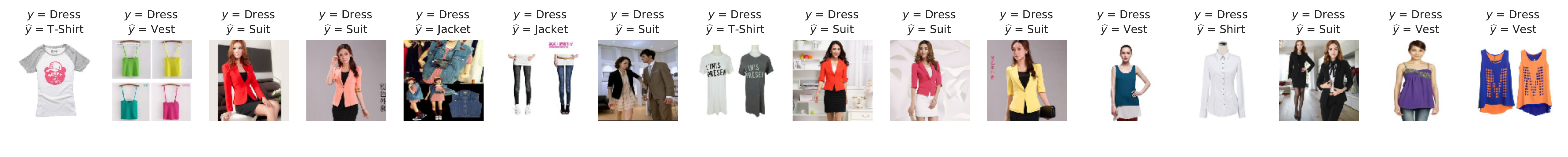}
    \end{subfigure}
    \begin{subfigure}{0.9\textwidth}
    \includegraphics[width=\textwidth]{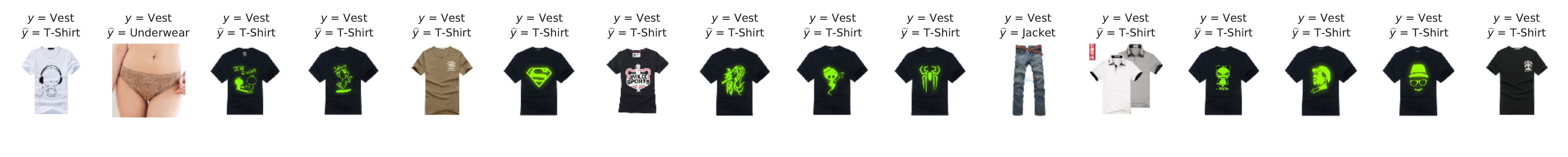}
    \end{subfigure}
    \begin{subfigure}{0.9\textwidth}
    \includegraphics[width=\textwidth]{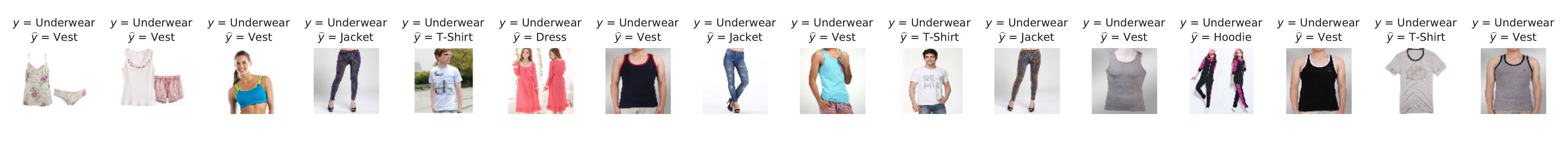}
    \end{subfigure}
    \caption{Most confusing 16 labels per class in the Clothing1M dataset, according to the distance between predicted and cross-entropy gradients. The gradient predictions are done using the best instance of LIMIT.}
    \label{fig:clothgin1m-confusing-more-examples}
\end{figure}

\end{document}